\documentclass[nohyperref]{article}

\usepackage{microtype}
\usepackage{graphicx}
\usepackage{booktabs} 

\usepackage{hyperref}

\usepackage[accepted]{icml2022}

\usepackage{amsmath}
\usepackage{amssymb}
\usepackage{mathtools}
\usepackage{amsthm}

\usepackage[capitalize,noabbrev]{cleveref}

\theoremstyle{plain}
\newtheorem{theorem}{Theorem}[section]

\newtheorem{lemma}[theorem]{Lemma}
\newtheorem{corollary}[theorem]{Corollary}
\theoremstyle{definition}

\theoremstyle{remark}

\usepackage[textsize=tiny]{todonotes}

\usepackage{subcaption}
\usepackage{bbm}
\usepackage{bm}
\usepackage{time}

\usepackage{pgfplots}

\icmltitlerunning{A General framework for PAC-Bayes Bounds for Meta-Learning}

\newcommand{\Btsk}{ \textup{B}_\textup{Task}}
\newcommand{\Benv}{ \textup{B}_\textup{Env}}

\newcommand{\taskRV}{\ensuremath{T}}
\newcommand{\task}{\ensuremath{t}}
\newcommand{\lenZm}{\ensuremath{\mathrm{M}}} 
\newcommand{  \lenDataSet}{\ensuremath{\mathrm{N}}}

\newcommand{\kposter}{\ensuremath{\kappa_s^2}}
\newcommand{\kprior}{\ensuremath{\kappa_p^2}}

\newcommand{\metaGenLoss}{\ensuremath{\mathrm L_{\mathrm P_{\taskRV\bm Z^{\lenZm}}}(u)}}

\newcommand{\Prior}{\ensuremath{\mathrm P}}
 \newcommand{\Posterior}{\ensuremath{\mathrm Q}}  
\newcommand{\hyperprior}{\ensuremath{\mathcal P}}
\newcommand{\hyperpost}{\ensuremath{\mathcal Q}}

  \newcommand{\aEnv}{\lambda_{\textup{Env}}}
   \newcommand{\aTsk}{\lambda_{\textup{Tsk}}}

\newcommand{\PZti}{\ensuremath{\mathrm P_{Z|\task_i}}}

\newcommand{\Dkl}{\ensuremath{\mathrm D}}

\newcommand{\Dgama}{\ensuremath{\mathrm D_{\gamma}}}

\newcommand{\trainTsk}{\ensuremath{\bm Z_i^{\lenZm}}}
\newcommand{\trainTskSimp}{\ensuremath{\bm Z_i^{\lenZm}}}
\newcommand{\trainEnvSimp}{\ensuremath{\bm Z_{1:\lenDataSet}^{\lenZm}}}

\newcommand{\trainEnv}{\ensuremath{\bm Z_{1:\lenDataSet}^{\lenZm}}}
\newcommand{\expLossTsk}[2][i]{\ensuremath{\mathrm L_{\PZti}(#1)}}

\newcommand{\DefexpLossTsk}[2][i]{\ensuremath{\mathbbm E_{\PZti}\left[ \ell(#1,#2)\right]}}

\newcommand{\expLossEnv}[1][i]{\ensuremath{\mathrm L_{\mathrm P_{\taskRV\bm Z^\lenZm}}(#1)}}
 
  \newcommand{\PWcondZimU}{\mathrm Q_i}
    \newcommand{\PWcondZimUsimp}{\ensuremath{\mathrm Q_i
    }
    }
      \newcommand{\PWcondZimUsimpSmallU}{\mathrm Q_i}

\newcommand{\empLossEnv}[1][i]{\ensuremath{\mathrm L_{\trainEnv}(#1)}}

\newcommand{\expLossTskRV}{\ensuremath{\mathrm L_{\PZti}(W)}}
\newcommand{\expLossTskRVRV}{\ensuremath{\mathrm L_{\mathrm P_{Z|\taskRV_i}}(W)}}
\newcommand{\empLossTskRV}{\ensuremath{\mathrm L_{\trainTsk}(W)}}
\newcommand{\empLossEnvRV}{\ensuremath{\mathrm L_{\trainEnv}(U)}}
\newcommand{\expLossEnvRV}{\ensuremath{\mathrm L_{\mathrm P_{\taskRV\bm Z^\lenZm}}(U)}}

 \newcommand{\decomMeta}[3][i]{\ensuremath{\tilde{\mathrm L}^{#3}_{#2}(#1)}} 
  \newcommand{\decomMetaRV}{\ensuremath{\decomMeta [U]{\bm Z^{\lenZm}_i}{\task_i}}} 
   \newcommand{\decomMetaRVTot}
   {\ensuremath{\frac{1}{\lenDataSet}\sum_{i=1}^{\lenDataSet} \decomMetaRV}}
      \newcommand{\decomMetaRVTotRV}
   {\ensuremath{\frac{1}{\lenDataSet}\sum_{i=1}^{\lenDataSet} \decomMeta [U]{\bm Z^{\lenZm}_i}{\taskRV_i}}}
  \newcommand{\DefdecomMeta}[3][i]{\ensuremath {\mathbbm E_{W\sim 
  \PWcondZimUsimpSmallU}\left[\mathrm L_{\mathrm P_{Z|#3}}(W)\right]}}
  
\newcommand{\func}[2][i]{\ensuremath{\mathrm F \left(#1,#2\right)}}

\newcommand{\funcTskRV}{\ensuremath{\mathrm F^{\textup{Task}} }}
\newcommand{\funcEnvRV}{\ensuremath{\mathrm F^{\textup{Env}} }}
\newcommand{\funcTsk}[2][i]{\ensuremath{\mathrm F^{\textup{Task}} \left(#1,#2\right)}}
\newcommand{\funcinvTsk}[1]{\ensuremath{\mathrm G_{\textup{Task}} \left(#1\right)}}
\newcommand{\funcinvEnv}[1]{\ensuremath{\mathrm G_{\textup{Env}} \left(#1\right)}}
\newcommand{\funcEnv}[2][i]{\ensuremath{\mathrm F^{\textup{Env}} \left(#1,#2\right)}}

\newcommand{\funcinv}[1]{\ensuremath{\mathrm G \left(#1\right)}}

\newcommand{\PriorTsk}{\Prior}
\newcommand{\PriorTskSimp}{\Prior}
\newcommand{\PriorTskSimpSmallU}{\mathrm P}
\newcommand{\JointPriorTskSimp}{\hyperprior\Prior}

\newcommand{\PriorEnvSimp}{\mathcal P}

\newcommand{\empLossTsksimple}[1][i]{\ensuremath{\mathrm L_{\trainTskSimp}(#1)}}

\usepackage[utf8]{inputenc}
\usepackage{array}
\usepackage{makecell}

\newcolumntype{P}[1]{>{\centering\arraybackslash}p{#1}}

\newcommand{\squeezeup}{\vspace{-2.5mm}}
\begin{document}

\onecolumn

\icmltitle{A General framework for PAC-Bayes Bounds for Meta-Learning}



\icmlsetsymbol{equal}{}

\begin{icmlauthorlist}
\icmlauthor{Arezou Rezazadeh}{yyy}
\end{icmlauthorlist}

\icmlaffiliation{yyy}{Department of Electrical Engineering,  Chalmers University of Technology, Gothenburg, Sweden}

\icmlcorrespondingauthor{Arezou Rezazadeh}{arezour@chalmers.se}





\footnote{
This work has been funded by the European Union’s Horizon 2020
research and innovation programme under the Marie Sklodowska-Curie grant
agreement No. 893082.\\
Department of Electrical Engineering,  Chalmers University of Technology, Gothenburg, Sweden.
arezour@chalmers.se}
 \begin{abstract}

Meta learning automatically infers an inductive bias, that includes the hyperparameter of the base-learning algorithm, by observing data from a finite number of related tasks. This paper studies PAC-Bayes bounds on meta generalization gap. The meta-generalization gap comprises two sources of generalization gaps: the environment-level and task-level gaps
resulting from observation of a finite number of tasks and data samples per task, respectively. In this paper, by upper bounding arbitrary convex functions, which link the expected and empirical losses at the environment and also per-task levels, we obtain new PAC-Bayes bounds. Using these bounds, we develop new PAC-Bayes meta-learning algorithms. Numerical examples
demonstrate the merits of the proposed novel bounds and algorithm in comparison to prior PAC-Bayes bounds for meta-learning.

\end{abstract}
\section{Introduction}
\label{br_int}
 Based on Mitchell's definition \cite{mitchel}, a machine learns a task from an experience when its performance improves with training examples of the task. In other words, during the learning process, the learner can produce a \textit{hypothesis}  that performs well on future examples of the same task. This learning process is done based on the set of assumptions known as \textit{inductive bias} \cite{baxter}. In many machine learning problems, finding methods for automatically learning the inductive bias is desirable. \textit{Meta learning} also known as \textit{learning to learn} \cite{thrun1998learning} formalizes this goal by observing data from a number of inherently related tasks. Then, it uses the gained experience and knowledge to learn appropriate bias which can be fine-tuned to perform well on new tasks.
Thus, the meta-learner speeds up the learning of a new, previously unseen task \cite{baxter}.
 For instance, learning the initialization and the learning rate of a training algorithm \cite{finn2017model,li2017meta}, the model architectures of a neural network \cite{Zoph2018},  or the optimization algorithm of a neural network  \cite{Ravi2017},  all are within the scope of meta-learning.

  As mentioned, the goal is extracting knowledge from several observed tasks referred to as \textit{meta-training set}, and using the knowledge to improve performance on a novel task. The meta-learner generalizes well if after observing sufficiently training tasks, it infers a \textit{hyperparameter} which contains good solutions to novel tasks. The good solution means that \textit{meta-generalization loss}, which is defined as the average loss incurred by the hyperparameter when used on a new task, is minimized. However, since both data and task
distributions are unknown, the meta-generalization loss can not be optimized. Instead, the meta-learner evaluates the \textit{empirical meta-training loss} for the hyperparameter based on the meta-training set. \textit{Meta-generalization gap} is defined as the difference between the meta-generalization loss and the meta-training loss. If the meta-generalization gap is small, it means that the meta-training loss is a good estimation of the meta-generalization loss.

Thus, bounding the meta-generalization gap is a key technique to understanding how the prior knowledge acquired from previous tasks may improve the performance of learning an unseen task. Here, a key question is `how to regularize the meta-learner, to avoid overfitting?'
The probably approximately correct (PAC)-Bayes generalization bound, is one way to answer this question.

In this paper, we derive a general framework that gives PAC-Bayes bounds on the meta-generalization gap. Under certain setups, different families of PAC-Bayes bounds, namely classic, quadratic and fast-rate families, can be re-obtained by the general framework. We also propose new PAC-Bayes classic bounds which reduce the meta-overfitting problem.  
\paragraph*{Related Work}
In  statistical meta-learning problems,  one line of research is learning of the parameters of the optimization algorithms, and analyzing gradients based on meta-learning methods  \cite{finn2017model,paper3}. For example, \cite{paper2,paper5} worked on an online convex optimization framework with the assumption that tasks are close to a global task parameter. Additionally, \cite{paper1,paper4} studied algorithms which incrementally update the bias regularization parameter using a sequence of observed tasks.
Another line of research is studying the meta-generalization gap, and finding bounds on it  on average \cite{sharu, arezou2021}
or with high probability \cite{pentina14, amit2018meta,rothfuss2020pacoh,liu2021,mys}.

We recall that in the ordinary learning problem, the bound for \textit{generalization gap} can be obtained for average  generalization error scenario \cite{Russo2016ControllingTheory,Xu2017,Bu,Negrea2019Information-TheoreticEstimates} and PAC-Bayes scenario \cite{McAllaster1999,Seeger2002,Maurer,catoni2007,Alquier2008,McAllester2013,Guedj2019,Guedj2019ALearning,dziugaite20-10b,Ohnishi2020,omar}.
In the former case, the bound of generalization error is derived by averaging over the training set and hypothesis. While, the PAC-Bayes bounds hold with high probability.

Following the initial work of  McAllester \cite{McAllaster1999},  PAC-Bayes bounds for conventional learning have been widely investigated. Selecting  different convex functions, which link the expected and empirical losses,  such as KL-divergence \cite{Seeger2002},  square function \cite{Mcallester2003} or linear function \cite{Alquier2016} implies  different PAC-Bayes  bounds.
The dependency on the sample size, in most of these bounds, is inversely proportional to the square root of the number of samples.
In \cite{McAllester2013}, by choosing the convex function as    $\Dgama(a||b)=\gamma a-\log(1-b+be^\gamma)$,  a family of PAC-Bayes bounds known as fast-rate bounds were obtained.
In these kinds of bounds, the dependence on the sample size can be improved  by the inverse of the number of samples. Directly relevant to this paper, in \cite{omar} by proposing a general approach of finding PAC-Bayes bounds, various known and also new PAC-Bayes bounds were obtained. 

In the meta-learning setup, inspired by the PAC-Bayes bounds for conventional learning problem, by using different convex functions, different kinds of bounds were obtained \cite{pentina14,amit2018meta,rothfuss2020pacoh,liu2021,mys}.
Initially, an extension of generalization error bounds to meta-learning was provided in \cite{pentina14} with a convergence rate $O(1/\sqrt{{\lenDataSet}})+O(1/(\lenDataSet\sqrt{{\lenZm}})+1/\sqrt{\lenZm})$. To have tighter bounds, the  approaches proposed in \cite{McAllaster1999} and \cite{Alquier2016} have been extended to the meta-learning problem in  \cite{amit2018meta} and \cite{rothfuss2020pacoh}, respectively.
  In \cite{amit2018meta} with a rate $O(\sqrt{\log(\lenDataSet)/{\lenDataSet}})+O(\sqrt{\log(\lenDataSet\lenZm)/{\lenZm}}$, by minimizing the obtained PAC-Bayes bound, a gradient-based algorithm was proposed. In \cite{rothfuss2020pacoh} with a convergence rate $O(1/\sqrt{{\lenDataSet}})+O(1/(\lenDataSet\sqrt{{\lenZm}})+1/\sqrt{\lenDataSet})$, by optimizing the obtained bound,  a class of PAC-optimal meta-learning algorithms was developed. To achieve meta-learning algorithms with rapid convergence ability, \cite{liu2021} and \cite{mys} have studied fast-rate bounds for the meta-learning setup with improved complexities.

\paragraph*{Contributions}
Here, we summarize the main contributions of the paper.
\begin{itemize}
    \item  Firstly, inspired by \cite{omar}, by upper bounding arbitrary convex functions, which link the expected and empirical losses at  environment and also per-task levels, we propose the general
PAC-Bayes meta-generalization bounds (Section \ref{gen_sec}). 

\item
Proper choices of the  convex functions recover known PAC-Bayes bounds including classic, quadratic and fast-rate families (Section \ref{pre_works}).

\item We  provide a new fast-rate bound and also a new classic bound  with better performance on the meta-test set and with convergence rate $O(\sqrt{(1/\lenDataSet+1/\lenZm)
})$ (Section \ref{our_bounds_sec}).
 Following the  meta-learning by adjusting the priors (MLAP) algorithm \cite{amit2018meta}, we develop the MLAP algorithm for  our new obtained bounds in the Section
 \ref{meta-alg}. 
 We demonstrate the usefulness of the proposed bounds in an example in Section  \ref{num_res}.
 The main merit of our new classic bound is its significant performance to avoid meta overfitting.

\end{itemize}

\section{ Notations, Definitions and Methods }
\label{lab1_int}
In this paper, the sample  $Z$ takes on a value in the instance
space $\mathcal Z$. The hypothesis space (named also as model
parameter space) is denoted by $\mathcal W$. The non-negative loss function $\ell : \mathcal W\times \mathcal Z\rightarrow \mathbbm R^+$ measures the model parameter $w\in\mathcal W$ on a datasample $z\in\mathcal Z$, the hyperparameter space
is represented by $\mathcal U$, and the task environment is defined by a set of tasks $\mathcal T$ which can be a discrete or a continuous set. The Kullback-Leibler (KL) divergence between two Bernoulli distributions with respective parameters $p$ and $q$, is given by $kl(p,q)$. In other cases, the KL divergence between distributions $\mathrm Q$ and $\mathrm P$ is denoted by $\Dkl(\mathrm Q||\mathrm P)$.

\subsection{Conventional single-task learning}
In conventional learning,  each task $\task\in\mathcal T$   is associated
with an underlying unknown data distribution  $\mathrm P_{Z|\taskRV=\task}$ on $\mathcal{Z}$.   For a given task
$\task_i\in \mathcal T$, the \textit{base-learner} observes a data set $\trainTsk=(Z^1_i,\dots,Z^{\lenZm}_i)$ of $\lenZm$ independently and identically distributed (i.i.d.) samples from
 $\mathrm P_{Z|\taskRV=\task_i}$. 
 For the
conventional single-task learning, the inductive bias comprising of the hyperparameter vector $u\in\mathcal U$ of the base-learner is
fixed. For the fixed $u\in\mathcal U$, 
the base-learner uses $u$ and the training set $\bm Z^\lenZm_i$ to  output a 
distribution 
over $\mathcal W$. 

The goal of the base-learner is to infer the model parameter  $w \in \mathcal{W}$ that minimizes the \emph{per-task generalization loss} (named also as the per-task expected loss)
\begin{align}
    \expLossTsk[w]{}= \DefexpLossTsk[w]{Z},\label{exp_loss_tsk}
\end{align} 
where the average is taken over a test sample $Z \sim \mathrm P_{Z|\taskRV=\task_i}$ drawn independently from $ \bm Z^\lenZm_i$. Since $\mathrm P_{Z|\taskRV=\task_i}$ is unknown, the generalization loss  $\mathrm L_{\mathrm P_{Z|\taskRV=\task_i}}(w)$ cannot be computed.
Instead, the base-learner evaluates the \emph{training loss}
\begin{align}
        \empLossTsksimple[w]{}= \frac{1}{\lenZm}\sum_{j=1}^{\lenZm}\ell(w,Z_i^j).
        \label{emp_loss_task}
\end{align} 
The difference between the generalization loss and the
training loss is referred to as the generalization gap
\begin{align}
   \Delta \mathrm L(w|\bm Z^{\lenZm}_i,u,\task_i)=\expLossTsk[w]{}-  \empLossTsksimple[w]{}.
       \label{2020.06.22_23.18}
\end{align} 

Roughly speaking, if the generalization gap is small, then with
high probability, the performance of the inferred model
parameter $w$ on the training set can be taken as a reliable measure of the per-task generalization loss. Here,
the question is that if we want to avoid overfitting and
minimize per-task generalization loss with respect to $w$, 
what should be optimized on the training data $\trainTsk$? The
PAC-Bayes framework studies this problem.

Given hyperparameter vector $u\in \mathcal U$,  and   task $\task_i\in \mathcal T$, 
in
the conventional single-task PAC-Bayes setting \cite{Alquier2021}, 
the base-learner assumes a prior distribution  $\PriorTskSimpSmallU$ over $\mathcal W$. By observing the training data $\trainTsk$, the base learner
updates the prior distribution to a data-dependent distribution referred as posterior distribution $\PWcondZimUsimpSmallU$. Having a new instance, the base learner randomly picks a model parameter $w\in\mathcal W$ according to $\PWcondZimUsimpSmallU$. 
To have a guarantee that the performance of training loss for the picked $w$ holds with
high probability as the performance of per-task generalization loss, we bound generalization gap averaged over
the posterior distribution, i.e., $\mathbbm E_{W\sim\PWcondZimUsimpSmallU}[\Delta \mathrm L(W|\bm Z^{\lenZm}_i,u,\task_i)]$.   

Roughly speaking, most PAC-Bayes proofs follow four key steps. \cite{Alquier2021} presents a comprehensive tutorial
about PAC-Bayes bounds. Here, we review the key steps of
finding PAC-Bayes bounds. 
 Let $\func[a]{b}$  be a convex function in both $a$ and $b$.
Firstly, a suitable convex function such as $\func[\cdot]{\cdot}$ links the expected loss  averaged over the posterior distribution with the empirical loss 
averaged over the posterior distribution. Then, by applying Jensen’s inequality, the function over the expectation (posterior distribution) is bounded by the expectation of the function.
By using a change of measure inequality \cite{Ohnishi2020}, we find a bound in terms of a divergence (usually KL-divergence between  posterior and prior distributions),    
and the expectation of the  function over prior distribution. Then, by applying Markov's inequality, we usually bound the expectation of the function with the logarithm
of the confidence parameter. 
Thus, the convex function linking the expected and empirical losses is bounded by a complexity term, like $\func[a]{b}\leq c$.
Usually, a further bounding technique, which we refer to as the ‘affine transformation’, is used to bound the expected loss as an affine transformation of the complexity term. It means that from  $\func[a]{b}\leq c$, one can conclude that $a\leq k\cdot b+\funcinv{c}$, where $k\in \mathbbm R$ is a coefficient, and $\mathrm G:\mathbbm R^+\rightarrow\mathbbm R$.
 
To look through the mentioned concepts in detail, we consider the conventional PAC-Bayes bound in \cite{McAllaster1999}. In \cite{McAllaster1999}, by setting $\funcTsk[a]{b}=2(\lenZm-1)(a-b)^2$,
it is proved that    given the prior distribution $\PriorTskSimpSmallU$, for any confidence parameter  $\delta\in(0,1)$, with probability at least $1-\delta$,
\begin{align}
 2(\lenZm-1) \left( \mathbbm E_{W\sim\PWcondZimUsimpSmallU}\left[ \expLossTsk[W]{}-\empLossTsksimple[W]\right]\right)^2 \leq \Dkl\left(\PWcondZimUsimpSmallU||\PriorTskSimpSmallU \right)+\log\left(\frac{\lenZm}{\delta}\right). \label{2021.06.15_16.30_neq}
\end{align}

The right hand side of \eqref{2021.06.15_16.30_neq}  is known as the \emph{complexity term}, and contains  KL-divergence, as the information gain in specializing from the prior to posterior distributions, and the log-term, as the dependence expression on the confidence parameter, and the number of
samples $\lenZm$.
A learning algorithm with generalization guarantee selects a posterior distribution $\PWcondZimUsimpSmallU$  which  minimizes \eqref{2021.06.15_16.30_neq}. Since minimizing \eqref{2021.06.15_16.30_neq} is not easy, to find bounds which are convenient to minimize, we apply the affine transformation step. In other words, for the convex function $\funcTsk[a]{b}=2(\lenZm-1)(a-b)^2$,
since
 from
$\funcTsk[a]{b}\leq c_{\textup{tsk}}$,
we have $a\leq 1.b+\sqrt{c_{\textup{tsk}}/(2(\lenZm-1))}$,  the affine transformation leads to
 $k_t=1$ and $\funcinvTsk{c}=\sqrt{c_{\textup{tsk}}/(2(\lenZm-1))}$. It means  that  from \eqref{2021.06.15_16.30_neq},
the following inequality holds uniformly for all posteriors distributions $\PWcondZimUsimpSmallU$ 
\begin{align}
   \mathbbm E_{W\sim\PWcondZimUsimpSmallU}\left[ \expLossTsk[W]{}-\empLossTsksimple[W]\right] \leq  \sqrt{\frac{\Dkl\left(\PWcondZimUsimpSmallU||\PriorTskSimpSmallU \right)+\log(\frac{\lenZm}{\delta})}{2(\lenZm-1)}}
   . \label{2021.06.15_16.30}
\end{align}

\subsection{Meta-Learning} 
The goal of meta-learning is  automatically infer the hyperparameter  $u$ of the base learner from training data pertaining to a number of related tasks. 
The tasks are assumed to belonging to a task environment, which is defined by a task distribution $\mathrm P_{\taskRV}$ on the space of tasks $\mathcal T$, and
by the per-task data distributions $\{\mathrm P_{Z|\taskRV=\task}\}_{\task \in \mathcal{T}}$.
The meta-learner observes a meta-training set $\trainEnv=(\bm Z_1^{\lenZm},\dots,\bm Z_\lenDataSet^{\lenZm})$ of $\lenDataSet$ data sets. Without loss of generality, we assume that
number of samples of all tasks equals to $\lenZm$. The obtained
results can be easily generalized to the case where per-task data samples are not equal.
Each $\trainTsk$ is generated independently by first drawing a task $\taskRV_i\sim \mathrm P_\taskRV$ and then a task-specific dataset  $\trainTsk\sim \mathrm P_{\bm Z^{\lenZm}|\taskRV_i}$.

The meta-learner uses the meta-training set $\trainEnv$ to infer the hyperparameter $u$.  
In the PAC-Bayes setup for meta
learning, the goal of the meta-learner is to infer hyperparameter $u$ 
from the observed tasks, and then use  $u$ as a
prior knowledge for learning new (yet unobserved) tasks
from task environment $\mathcal T$. The  quality of $u$ is is measured
by the \emph{meta-generalization loss} when using it to learn
new tasks. Formally, the objective of the meta-learner
is to infer the hyperparameter $u$ that minimizes the meta-generalization loss
\begin{align}
    \metaGenLoss=\mathbbm E_{\mathrm P_{\taskRV} \mathrm P_{\bm Z^{\lenZm} |\taskRV}} \bigl[
    \mathbbm E_{W\sim \Posterior}
 [\mathrm L_{\mathrm P_{Z|\taskRV}}(W)] \bigr],
    \label{sharu_eq12}
\end{align}
where the expectation is taken over an independently
generated meta-test task $\taskRV \sim \mathrm P_\taskRV$, over the associated data set $\bm Z^\lenZm \sim \mathrm P_{\bm Z^\lenZm|\taskRV}$, and over the output of the base-learner. 
Since $\mathrm P_\taskRV$ and $\{\mathrm P_{Z|\taskRV=\task}\}_{\task \in \mathcal{T}}$ are unknown, the meta-generalization loss~\eqref{sharu_eq12}
cannot be computed. Instead, the
meta-learner can evaluate the \emph{meta-training loss}, which
for a given hyperparameter $u$, is defined as the average
training loss on the meta-training set
\begin{align}
\mathrm L_{\trainEnv}(u)  =\frac{1}{\lenDataSet}\sum_{i=1}^{\lenDataSet}\mathbbm E_{W\sim\PWcondZimUsimpSmallU} [\mathrm L_{\bm Z^{\lenZm}_i}(W)].
\label{sharu_eq11}
\end{align}
Here, the average is taken over the output of the base-learner. The difference between meta-generalization
loss and the meta-training loss is the \emph{meta-generalization gap}
\begin{align}
\Delta\mathrm L(u|\trainEnv)=\metaGenLoss-\mathrm L_{\trainEnv}(u).
\label{2020.06.24_22.18}
\end{align}
Small meta-generalization gap means that with high probability, the performance of the inferred hyperparameter $u$ on the meta-training set can be taken as a reliable measure
of the meta-generalization loss~\eqref{sharu_eq12}.

In the PAC-Bayes setup for meta learning, the meta-
learner assumes a \emph{hyper-prior distribution} $\hyperprior\in\mathcal P_{\mathcal U}$ over hyperparameter space $\mathcal U$, observes the meta-training set $\trainEnv$, and updates the hyper-prior distribution to a data-dependent distribution referred as  \emph{hyper-posterior distribution}
 $\hyperpost\in\mathcal P_{\mathcal U}$. The goal is to use the hyper-posterior
distribution for learning new and unseen tasks. In other
words, having a new task, the meta learner randomly picks $u$ according to hyper-posterior distribution  $\hyperpost$ and then use it for learning of posterior $\PWcondZimUsimpSmallU$.

One approach for finding the PAC-Bayes bounds for meta
learning, is decomposing the meta-generalization gap
into environment-level and within-task generalization
gaps. We define the decomposition term as
\begin{align}
 \decomMetaRVTot
    ,\label{decom-pac_tot}
\end{align}
where $\decomMeta[u]{\bm Z^{\lenZm}_i}{\task_i}$ is the average per-task generalization loss
\begin{align}
    \decomMeta[u]{\bm Z^{\lenZm}_i}{\task_i}=\DefdecomMeta[u]{\bm Z^{\lenZm}_i}{\task_i}.\label{decom-pac}
\end{align}

From \eqref{sharu_eq12}, we can express the meta-expected loss as  $    \expLossEnv[U]{}=\mathbbm E_{\mathrm P_{T\bm Z^\lenZm}}[  \decomMeta[U]{\bm Z^\lenZm}{\taskRV}]$.
 Recalling that  $\funcEnv[a]{b}$ is  a convex function in both $a$ and $b$.
In the PAC-Bayes setup for meta learning, we can
follow the mentioned four steps for both environment-level generalization gap
\begin{align}
  \hspace{-0.7em}  \mathrm F^{\textup{Env}} \Bigg( \mathbbm E_{U\sim \hyperpost} \mathbbm E_{\mathrm P_{T\bm Z^\lenZm}}\left[  \decomMeta[U]{\bm Z^\lenZm}{\taskRV}\right], \mathbbm E_{U\sim\hyperpost} \left[\frac{1}{\lenDataSet}
\sum_{i=1}^\lenDataSet
    \decomMeta[U]{\bm Z^{\lenZm}_i}{\task_i}\right]\Bigg),\label{2021.06.16.00.57}
\end{align}
and within-task generalization gap 
\begin{align}
   \mathrm F^{\textup{Task}}\Bigg(\mathbbm E_{U\sim\hyperpost}\bigg(
   \frac{1}{\lenDataSet}\sum_{i=1}^{\lenDataSet} \mathbbm E_{ W\sim  \PWcondZimU}\Big(\mathrm L_{\mathrm P_{Z|\taskRV_i}}(W)\Big)
   \bigg), \mathbbm E_{U\sim\hyperpost} \bigg(
   \frac{1}{\lenDataSet}\sum_{i=1}^{\lenDataSet}\mathbbm E_{W\sim  \PWcondZimU} [\mathrm L_{\bm Z^{\lenZm}_i}(W)]
   \bigg)\Bigg),
   \label{2021.06.16.00.58}
\end{align}
 separately \cite{pentina14,amit2018meta,rothfuss2020pacoh,liu2021,mys}. 
\section{General Meta-Learning PAC-Bayes Bounds}
\label{gen_sec}
In this section, inspired by \cite{omar},  
we find
a general approach for finding PAC-Bayes bounds for meta-
generalization gap.
\begin{theorem}[General PAC-Bayes Bounds]
\label{prop-gen-pac-tsk_two_KL}
Let $\funcTsk[a]{b}$ and $\funcEnv[a]{b}$ be two functions which are convex in both $a$ and $b$. Additionally, assuming that the tasks are drawn
independently from the task environment $\mathcal T$ according to distribution $\mathrm P_{\taskRV}$.  For the task and environment level priors $\PriorTskSimp$ and $\PriorEnvSimp$, with a probability at least $1-\delta$, under $\mathrm P_{T_{1:\lenDataSet}}\mathrm P_{\trainEnvSimp|T_{1:\lenDataSet}}$,
for $\theta_{\textup{tsk}},\theta_{\textup{env}}\geq 0$ we have
\begin{align}
&\funcEnv[\mathbbm E_{U\sim\hyperpost}\left(
\expLossEnvRV
\right)]{\mathbbm E_{U\sim\hyperpost}\left( \decomMetaRVTotRV\right)}
+\funcTsk[\mathbbm E_{U\sim\hyperpost}\big( \decomMetaRVTotRV\big)]{\mathbbm E_{U\sim\hyperpost} \big(\empLossEnvRV\big)}\nonumber\\
&\hspace{-.5em} \leq 
\left( \frac{1}{\theta_{\textup{tsk}}}+\frac{1}{\theta_{\textup{env}}}\right)\Dkl\left(\hyperpost||
     \PriorEnvSimp
     \right)+ \frac{1}{\lenDataSet\cdot\theta_{\textup{tsk}}}
     \mathbbm E_{\hyperpost}\left(\sum_{i=1}^{\lenDataSet}\Dkl\left(\PWcondZimUsimp||
     \PriorTskSimp
     \right)\right) +  \log\frac{\mathbbm E_{\mathrm P_{T_{1:\lenDataSet}}}\mathbbm E_{\mathrm P_{\trainEnvSimp|T_{1:\lenDataSet}}}\left(\Upsilon_{\textup{tsk}}^{\frac{1}{\lenDataSet\cdot\theta_{\textup{tsk}}}}\cdot
     \Upsilon_{\textup{env}}^{\frac{1}{\theta_{\textup{env}}}}
\right)}{\delta},\label{cor_tup_1}
\end{align}
where
\begin{align}
\Upsilon_{\textup{env}}&=\mathbbm E_{\PriorEnvSimp}e^{\theta_{\textup{env}}
\funcEnv[
\expLossEnvRV
]{\decomMetaRVTotRV}}\label{up_env}\\
\Upsilon_{\textup{tsk}}&=
     \prod_{i=1}^{\lenDataSet}\mathbbm E_{ \hyperprior\PriorTsk} e^{ \theta_{\textup{tsk}}
     \funcTsk[
     \mathrm L_{\mathrm P_{Z|\taskRV_i}}(W)]{
     \empLossTskRV}}\label{up_tsk}.
\end{align}
\end{theorem}
\begin{proof}
See Appendix \ref{Proof_prop-gen-pac-tsk_two_KL}.
\end{proof}

     Generally, to obtain  \eqref{cor_tup_1}, we applied only one Markov’s
inequality. Thus, on the left hand side of \eqref{cor_tup_1},   we have the sum of $\funcEnvRV(\cdot)$ and $\funcTskRV(\cdot)$.
A relaxed form of \eqref{cor_tup_1}, can be obtained by applying the affine transformation
and also Markov’s inequality two times at the task and
environment levels, separately. As discussed in \eqref{2021.06.15_16.30}, the
affine transformation leads to a new function denoted by $\funcinv{}$.
For example, if the convex function is $\func[a]{b}=(a-b)^2$, since from $\func[a]{b}=(a-b)^2\leq c$, we conclude that $a\leq b+\sqrt{c}$, the affine transformation leads to $k=1$ and $\funcinv{c}=\sqrt{c}$. Similarly, $\func[a]{b}=kl(a,b)\leq c$ leads to $k=1/(1-0.5\lambda)$, and $\funcinv{c}=c/(\lenZm.\lambda(1-0.5\lambda))$  \cite{Thiemann2017} .
 The
following corollary is a relaxation of  \eqref{cor_tup_1}.
\begin{corollary}
 Under the setting of Theorem \ref{prop-gen-pac-tsk_two_KL}, assume that   $\funcinvTsk{\cdot} $ and $\funcinvEnv{\cdot} $ are two functions where from   $\funcTsk[a]{b}\leq c_{\textup{tsk}}$ (res.  $\funcEnv[a]{b}\leq c_{\textup{env}}$), we can conclude 
$a\leq k_t\cdot b+ \funcinvTsk{c_{\textup{tsk}}}$ (resp. $a\leq k_e\cdot b+\funcinvEnv{c_{\textup{env}}}$) for $k_t\in\mathbbm R^+$ (resp. $k_e\in\mathbbm R^+$).
In this case, with probability at
least $1-\delta$, under  $\mathrm P_{T_{1:\lenDataSet}}\mathrm P_{\trainEnvSimp|T_{1:\lenDataSet}}$, 
we have
       \begin{align}
\mathbbm E_{U\sim\hyperpost}\left[\expLossEnvRV\right]\leq k_e\cdot k_t\cdot\mathbbm E_{U\sim\hyperpost}\left[\empLossEnv[U]\right]+
\funcinvEnv{
\Benv}%
+ \frac{k_e}{\lenDataSet}\sum_{i=1}^\lenDataSet\funcinvTsk{\Btsk}
,\label{cor_tup_2}
 \end{align}
  where 
     \begin{align} \Benv=\frac{1}{\theta_{\textup{env}}}&\Dkl\left(\hyperpost||
     \PriorEnvSimp
     \right)+\log \left(\frac{2\mathbbm E_{\mathrm P_{T_{1:\lenDataSet}}}\mathbbm E_{\mathrm P_{\trainEnvSimp|T_{1:\lenDataSet}}}\mathbbm E_{\PriorEnvSimp}e^{\theta_{\textup{env}}\funcEnv[
\expLossEnvRV
]{\decomMetaRVTotRV}}}{\delta}\right)^{\frac{1}{\theta_{\textup{env}}}},\label{B_env_eq}
\end{align}
and  
\begin{align}
         \Btsk=\frac{1}{ \theta_{\textup{tsk}}}
         \Dkl\left(\hyperpost||
     \PriorEnvSimp 
     \right) +\frac{1}{ \theta_{\textup{tsk}}}\mathbbm E_{\hyperpost}\left[\Dkl\left(
     \PWcondZimUsimp||\PriorTskSimp\right)\right] 
     +\log\left(
\frac{ 2\lenDataSet \mathbbm E_{\mathrm P_{T_{1:\lenDataSet}}}\mathbbm E_{\mathrm P_{\trainEnvSimp|T_{1:\lenDataSet}}}  \mathbbm E_{   \PriorEnvSimp \PriorTskSimp} e^{\theta_{\textup{tsk}}\funcTsk[ \mathrm L_{\mathrm P_{Z|\taskRV_i}}(W)]{\empLossTskRV}}}{\delta}
     \right)^{\frac{1}{ \theta_{\textup{tsk}}}}.\label{B_task_eq}
     \end{align}
    
\end{corollary}
\begin{proof}
See Appendix \ref{Proof_prop-gen-pac-tsk_two_KL}.
\end{proof}

\section{Re-obtaining Existing Results}
\label{pre_works}

 \begin{table*}[t]
\centering
\caption{ Existing PAC-Bayes bounds on meta generalization gap can be obtained as a special case of \eqref{cor_tup_2}}. 
\rule{\linewidth}{0cm}
\begin{tabular}{ |p{3cm}|p{5cm}|p{2.5cm}|p{4.7cm} | }
 \hline
 \multicolumn{4}{|c|}{Meta-Learning  PAC Bayes Bounds} \\
 \hline
Bound&  $\funcTsk[a]{b}$ , $\funcEnv[a]{b}$ & Other parameters& Affine transformation  \\
\hline
MLAP \newline\cite{amit2018meta}   &$\funcTsk[a]{b}=2(\lenZm-1)(a-b)^2$,\newline $\funcEnv[a]{b}=2(\lenDataSet-1)(a-b)^2$    & $\theta_{\textup{tsk}}=\theta_{\textup{env}}=1$&$\funcinvEnv{c}=\sqrt{c/(2(\lenDataSet-1))}$, $k_e=1$, $\funcinvTsk{c}=\sqrt{c/(2(\lenZm-1))}$, $k_t=1$\\
 \hline
 PACOH \newline\cite{rothfuss2020pacoh} & $\funcTsk[a]{b}=(a-b)$,\newline$\funcEnv[a]{b}=(a-b)$ & $\theta_{\textup{tsk}}=\beta$,  $\theta_{\textup{env}}=\lambda$  &  $\funcinvEnv{c}=c$, $k_e=1$, $\funcinvTsk{c}=c$, $k_t=1$\\
 \hline
   $\lambda$-Bound \newline\cite{liu2021} &$\funcTsk[a]{b}=\lenZm kl(a,b)$,\newline $\funcEnv[a]{b}=2(\lenDataSet-1)(a-b)^2$   &  $\theta_{\textup{tsk}}=\theta_{\textup{env}}=1$,\newline$\lambda\in(0,2) $&$\funcinvEnv{c}=\sqrt{c/(2(\lenDataSet-1))}$, \newline$k_e=1$, $k_t=1/(1-0.5\lambda)$,\newline$\funcinvTsk{c}= c/(\lenZm\lambda(1-0.5\lambda))$\\
 \hline
Classic bound \newline \cite{mys} &$\funcTsk[a]{b}=\lenZm kl(a,b)$ ,\newline $\funcEnv[a]{b}=\lenDataSet kl(a,b)$& 
$\theta_{\textup{tsk}}=\theta_{\textup{env}}=1$&$k_t=k_e=1$,\newline$\funcinvTsk{c}=\sqrt{c/2\lenZm}$,\newline$\funcinvEnv{c}=\sqrt{c/2\lenDataSet}$\\
\hline
Quadratic bound \newline\cite{mys} &$\funcTsk[a]{b}=\lenZm kl(a,b)$ ,\newline $\funcEnv[a]{b}=\lenDataSet kl(a,b)$&
$\theta_{\textup{tsk}}=\theta_{\textup{env}}=1$
&$k_e=1$,\newline$\funcinvEnv{c}=\sqrt{c/2\lenDataSet}$,\newline$b\leq (\sqrt{a+(c/2)}+\sqrt{c/2})^2$\\
 \hline
$\lambda$ bound \newline\cite{mys} &$\funcTsk[a]{b}=\lenZm kl(a,b)$ ,\newline $\funcEnv[a]{b}=\lenDataSet kl(a,b)$&$\lambda\in(0,2)$\newline
$\theta_{\textup{tsk}}=\theta_{\textup{env}}=1$
&$k_t1/(1-0.5\lambda)$,\newline
$\funcinvTsk{c}= c/(\lambda(1-0.5\lambda))$\newline
$k_e=1$, $\funcinvEnv{c}=\sqrt{c/2\lenDataSet}$\\
 \hline
\end{tabular}
\label{table_reobtain}
\end{table*}

In this section, by applying different $\funcEnv[\cdot]{\cdot}$ and $\funcTsk[\cdot]{\cdot}$ to \eqref{cor_tup_2}, we re-obtain all the previous main PAC-Bayes
bounds for the meta-learning problem. Table \ref{table_reobtain}, summarizes the  results. For the derivation see Appendix \ref{derv_previous}.

\section{New PAC-Bayes Bounds for Meta-Learning}
\label{our_bounds_sec}
In this section, we insert different  $\funcEnv[]{}$ and $\funcTsk[]{}$ in 
 \eqref{cor_tup_1}, and then we bound  \eqref{up_env} and \eqref{up_tsk} by using Lemmas presented in Appendix \ref{gen_lemma}.  For simplicity, we assume that
the loss function is bounded on the interval $[0,1]$, and
hence it is sub-Gaussian with parameter $0.5$. 
  
Mainly, we present a new fast-rate bound and a new clas-
sic bound. To obtain fast-rate bound, like existing bounds,
we use  \eqref{cor_tup_2}. It means that we apply Markov’s inequality
and affine-transformation steps in both the environment and
task levels. However, to find the classic bound, we find a
lower bound for the left-hand side of \eqref{cor_tup_1}, and then we
apply the affine-transformation step. It means that to obtain new classic bound, we apply both Markov’s inequality
and affine-transformation step, once. Firstly, we preset
the fast-rate bound.
 \begin{theorem}
\label{fast-rate-bnd1}
 Under the setting of Theorem \ref{prop-gen-pac-tsk_two_KL}, for $\lenDataSet\geq 2$, the meta-generalization gap is bounded by
 \begin{align}
  \mathbbm E_{\hyperpost} \left[
   \expLossEnvRV\right]\leq
   \min_{\lambda_e,\lambda_t\geq 0.5}
      \frac{1}{1-\frac{1}{2\lambda_e}}\cdot \frac{\mathbbm E_{\hyperpost}\empLossEnvRV}{1-\frac{1}{2\lambda_t}}
     +\frac{\lambda_e}{1-\frac{1}{2\lambda_e}}\left(
\frac{   \Dkl\left( \hyperpost ||   \PriorEnvSimp\right) +\log\frac{2}{\delta}}{\lenDataSet}\right)
      \nonumber\\+\frac{1}{1-\frac{1}{2\lambda_e}}\cdot
      \frac{\lambda_t}{1-\frac{1}{2\lambda_t}}\cdot
      \frac{1}{\lenDataSet}\sum_{i=1}^{\lenDataSet}\frac{    \Dkl\left(\hyperpost||
     \PriorEnvSimp 
     \right) +
    \mathbbm E_{\hyperpost}\left[\Dkl\left(
     \PWcondZimUsimp||\PriorTskSimp\right)\right]+
    \log\frac{2\lenDataSet }{\delta}}{\lenZm}
,
\label{fast_rate}
 \end{align}
 where \eqref{fast_rate} is referred as the fast-rate bound for meta-
learning.
\end{theorem}
\begin{proof}
  See Appendix \ref{Proof_fast-rate-bnd1}
\end{proof}

 Now, we obtain a new classic bound. Setting $\funcTsk[a]{b}=2(\lenZm-1)(a-b)^2$ and $\funcEnv[a]{b}=(\lenDataSet-1)(a-b)^2$  in \eqref{cor_tup_1}, we will obtain a new bound with a single square, unlike
existing bounds. The key step to find bounds with a single
square is the following inequality
 \begin{align}
      \frac{nm}{n+m}(a-c)^2\leq n(a-b)^2+ m(b-c)^2, \label{inv_ineq}
 \end{align}
 where $n,m\in\mathbbm N$.
 To show \eqref{inv_ineq}, consider the function $f(a,b,c)\triangleq n(a-b)^2+ m(b-c)^2$. Since $\partial^2 f/\partial b^2=2(n+m)>0 $, the function $f$ is convex with respect to $b$. Hence, 
 by setting the first derivative of $f$ with respect to $b$ equal to zero, 
  $b^\star=(n a+m c)/(n+m)$ minimizes $f$. Since  $f(a,b^\star,c)\leq f(a,b,c)$, and  $f(a,b^\star,c)$ equals to the left-hand side of  \eqref{inv_ineq}, we conclude  \eqref{inv_ineq}.
 
Now, in  \eqref{inv_ineq},  we set 
$
     a=\mathbbm E_{U\sim\hyperpost}(
\expLossEnvRV
)$,
$b=\mathbbm E_{U\sim\hyperpost}( \decomMetaRVTot)$ and 
$c=\mathbbm E_{U\sim\hyperpost} (\empLossEnvRV)$. 
 usually the right hand side of \eqref{inv_ineq} gives us the left hand side of \eqref{cor_tup_1}. 
 Thus, if the right hand side of \eqref{inv_ineq} is upper bounded by $\mathrm B$, from  \eqref{inv_ineq} we conclude that 
 \begin{align}
     \left |\mathbbm E_{U\sim\hyperpost} \left[
   \expLossEnvRV- \empLossEnvRV \right] \right|\leq         \sqrt{\frac{n+m}{nm}\mathrm B}.\label{inv_ineq2}
 \end{align}
Now, in view of \eqref{inv_ineq},
we insert various convex functions to \eqref{cor_tup_1}. In the following, we present one of them where we set  $\funcTsk[a]{b}=2(\lenZm-1)(a-b)^2$ and $\funcEnv[a]{b}=(\lenDataSet-1)(a-b)^2$ in \eqref{cor_tup_1}, and $n=(\lenDataSet-1)$, $m=2(\lenZm-1)$, in \eqref{inv_ineq}.

\begin{theorem}
\label{the_Nw_McAll}
 Under the setting of Theorem \ref{prop-gen-pac-tsk_two_KL}, $\lenDataSet\geq 2$, the meta-generalization gap is bounded by
 \begin{align}
    \left |\mathbbm E_{U\sim\hyperpost} \left[
   \expLossEnvRV- \empLossEnvRV \right] \right|
   \leq
   \sqrt{\frac{(\lenDataSet-1)+2(\lenZm-1)}{2(\lenDataSet-1)(\lenZm-1)}}
  \cdot \sqrt{2\Dkl\left(\hyperpost||\PriorEnvSimp  \right)
+\mathbbm E_{\hyperpost}\left[ \frac{1}{\lenDataSet}\sum_{i=1}^\lenDataSet \Dkl\left(\PWcondZimUsimp||\PriorTskSimp \right)\right]
+\log\frac{\lenZm\sqrt{\lenDataSet}}{\delta}}.
\label{New_McAll}
 \end{align}
\end{theorem}
 \begin{proof}
 See Appendix \ref{Proof_the_Nw_McAll}.
 \end{proof}
 We recall that \eqref{New_McAll}
 is expressed in terms of a single square.
Thus, compared to existing bounds, minimizing \eqref{New_McAll} is
less complicated and one can enjoy the properties of presented bounds in \cite{rothfuss2020pacoh}. In other words, it
may reduce the problem of meta overfitting.

In this sections, we only presented a member of fast-rate
and a member of classic families. We can apply different  $\funcEnv[a]{b}$ and $\funcTsk[a]{b}$ functions, and obtain new
different bounds. For more new bounds, see Appendix \ref{aux_app_for_new_bnd}.

\section{Meta-Learning Algorithm}
\label{meta-alg}

 \begin{table*}[t]
\centering
\caption{Comparison of different PAC-Bayes bounds on $20$  test tasks. During the meta-training phase, each task is constructed
with $8000$ images, and during the meta-test phase, each task is constructed with $2000$ images. The number  of training task is set as $\lenDataSet=5$. The number of epochs is $100$.  
}
\rule{\linewidth}{0cm}
\centering
\begin{tabular}{ |p{4cm}|p{4cm}|p{4cm}|p{4cm}|p{4cm} | }
 \hline
 \multicolumn{3}{|c|}{Meta-Learning  PAC Bayes Bounds: Test Error $(\%)$} \\
 \hline
Bound& $100$-Swap Shuffled pixels  & Permuted labels\\
 \hline
MLAP \cite{amit2018meta}   & $31.4$ with STD $ 1.99\%$      &$ 7.95$ with STD $ 0.441\%$ \\
 \hline
  kl-Bound \cite{liu2021} &$26.7$ with STD $ 1.14\%$ &     $9.86$ with STD $ 0.915\%$ \\
 \hline
  $\lambda$-Bound \cite{liu2021} &$28.1$ with STD $ 1.28\%$       &$9.8$ with STD $ 0.859\%$  \\
 \hline
 Bound 1 \cite{mys} &$31.8$ with STD $ 1.31\%$     &$8.026$ with STD $ 0.356\%$ \\
  \hline
 Bound 2 \cite{mys} & $32$ with STD $ 1.44\%$    &$11.06$ with STD $ 0.588\%$ \\
  \hline
 Bound 3 \cite{mys} &  $ 28.8$ with STD $ 1.6\%$  &$10.3$ with STD $ 0.784\%$ \\
 \hline
 Our classic bound  \eqref{New_McAll} &$9.39$ with STD $0.454\%$  &$3.1$ with STD $0.279\%$\\
 \hline
  Our fast-rate bound  \eqref{fast_rate} &$25.2$ with STD $1.23\%$  &  $ 9.77$ with STD $ 0.824\%$\\
  \hline
\end{tabular}
\label{table_result}
\squeezeup
\end{table*}
By minimizing the bounds obtained in Section \ref{our_bounds_sec}, we can
develop PAC-Bayes bound-minimization algorithms for
meta-learning with deep neural networks. Here, for \eqref{New_McAll}, and in view of  \cite{amit2018meta} we obtain a meta-learning algorithm. Similar algorithms can be found for
all bounds presented in Section \ref{our_bounds_sec}.
Following previous
work \cite{amit2018meta}, we consider $\mathcal W=\{h_w,w\in\mathbbm R^d\}$
as a set of neural networks with certain parameters. For  $\mathrm N_p, d\in \mathbbm N$, let $\mathcal U\subset \mathbbm R^{\mathrm N_p}$ and $\mathcal W\subset \mathbbm R^d$ for all tasks $\task_i$. Like previous works,    \cite{amit2018meta,mys}, 
we
set the hyper-prior distribution as
\begin{align}
    \hyperprior\triangleq \mathcal N(0,\kprior \mathrm I_{\mathrm N_p}\times\mathrm I_{\mathrm N_p}), \label{gua_hyper_prior}
\end{align}
where  $\kprior>0$ is a predefined constant. We limit the space
of hyper-posteriors as a family of isotropic Gaussian distributions defined by
\begin{align}
    \hyperpost\triangleq \mathcal N(\theta,\kposter \mathrm I_{\mathrm N_p}\times \mathrm I_{\mathrm N_p}),\label{gua_hyper_posterior}
\end{align}
where  $\theta\in \mathbbm R^{\mathrm N_p}$ is the optimization parameter,
and $\kposter>0$ is a predefined constant.

Next, we consider the posterior and prior distributions
over $\mathcal W$. For all tasks $\task_i\in\mathcal T$,  $\mathcal W$ can be seen as a family of functions parameterized by a weight vector  $\bm a^d=[a_1,\dots,a_d]$. 
For a given hyperparameter $u$,  let the weight
vector is denoted by  $\bm a^d$. We define the prior distribution
as factorized Gaussian distributions
\begin{align}
\PriorTsk(\bm a^d|u)=\prod_{k=1}^d \PriorTsk( a_k|u)=
\prod_{k=1}^d
\mathcal N(a_k;\mu_u(k),\sigma_u^2(k)),
    \label{gua_prior}
\end{align}
meaning that, before observing data, the $k$-th   weight denoted by $a_k$,  takes values according to Gaussian distribution with mean $\mu_u(k)$ and variance $\sigma_u^2(k)$.
 After observing
data, for task $\task_i\in \mathcal T$, 
$a_k$  takes values according to Gaussian
distribution with mean $\mu_i(k)$ and variance $\sigma_i^2(k)$. Thus,
the posterior distribution of task $\task_i $ is
\begin{align}
   \PWcondZimU(\bm a^d|\bm z_i^{\lenZm_i},u)=\prod_{k=1}^d\mathcal N(a_k;\mu_i(k),\sigma_i^2(k)). \label{gua_posterior}
\end{align}

From \eqref{gua_hyper_prior} and \eqref{gua_hyper_posterior}, it can be verified that
\begin{align}
    \Dkl\left(
    \hyperpost|| \hyperprior
    \right)=\frac{||\theta||^2_2+\kposter}{2\kprior}+\log\frac{\kprior}{\kposter}-\frac{1}{2}.\label{2023.06.17_17.12}
\end{align}
Similarly, from \eqref{gua_prior} and \eqref{gua_posterior}, for task $\task_i$, we find that
\begin{align}
    \Dkl\left(
   \PWcondZimU|| \PriorTsk
    \right)=\frac{1}{2}\sum_{k=1}^{d}\Bigg(\log\frac{\sigma_u^2(k)}{\sigma_i^2(k)}
    +\log\frac{\sigma_i^2(k)+\left(\mu_i(k)-\mu_u(k)\right)^2}{\sigma_u^2(k)}\Bigg).\label{2023.06.17_23.17}
\end{align}
Inserting \eqref{2023.06.17_17.12} and \eqref{2023.06.17_23.17} into \eqref{New_McAll}, it remains to select the
parameters of posterior distribution $\PWcondZimU$
minimizing  \eqref{New_McAll}. Since the square root function is strictly increasing, an equivalent optimization problem is the minimization of the
objective function inside the square of \eqref{New_McAll}. 

Following the optimization technique described in  \cite{amit2018meta}, approximating the expectation  $\mathbbm E_{U\sim\mathcal N(\theta,\kposter \mathrm I_{\mathrm N_p}\times \mathrm I_{\mathrm N_p})}$ by averaging several
Monte-Carlo samples of $U$, the optimal posterior distribution can be obtained by evaluating the gradient of
 \eqref{New_McAll}
with respect to $(\bm \mu_i, \bm \sigma_i^2) $ as described in Section 4.4 of \cite{amit2018meta}.

We recall that, like  \cite{rothfuss2020pacoh}, the minimization problem of \eqref{New_McAll} is equivalent to the simpler problem
than the optimization of existing classic bounds. In fact, it
suffices to minimize an objective function, which is linear
with respect to KL-divergences. This leads to Gibbs posteriors, and might be the reason why the obtained algorithm
reduces the meta-overfitting problem. 

\section{Numerical Results}
\label{num_res}

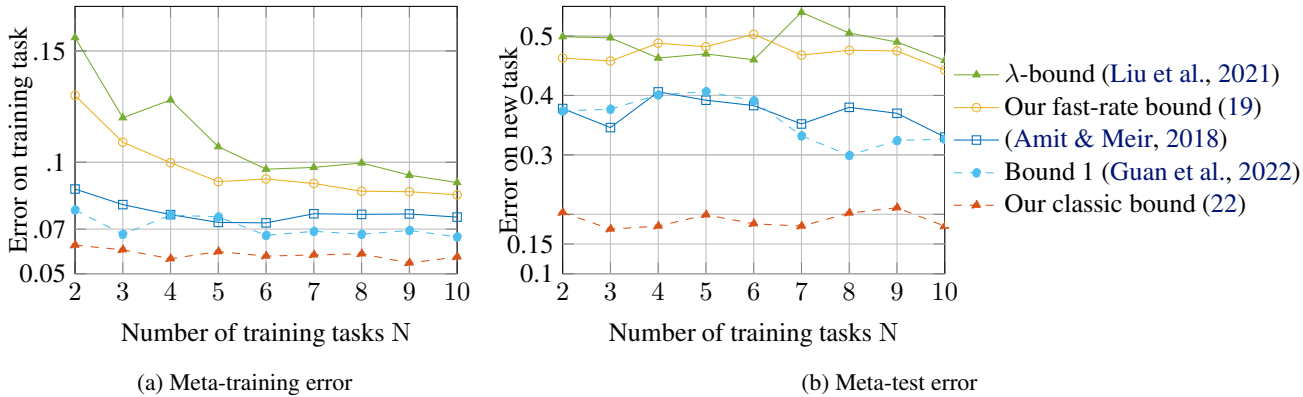
\begin{figure*}[!t]%
\centering
\begin{subfigure}[t]{.38\textwidth}
%
%
\definecolor{mycolor1}{rgb}{0.00000,0.44700,0.74100}%
\definecolor{mycolor2}{rgb}{0.85000,0.32500,0.09800}%
\definecolor{mycolor3}{rgb}{0.92900,0.69400,0.12500}%
\definecolor{mycolor4}{rgb}{0.49400,0.18400,0.55600}%
\definecolor{mycolor5}{rgb}{0.46600,0.67400,0.18800}%
\definecolor{mycolor6}{rgb}{0.30100,0.74500,0.93300}%
\definecolor{mycolor7}{rgb}{0.63500,0.07800,0.18400}%
\begin{tikzpicture}

\begin{axis}[%
width=2in,
height=1.4in,
at={(0.758in,0.481in)},
scale only axis,
xmin=2,
xmax=10,
ymin=0.05,
ymax=0.17,
xlabel=Number of training tasks $\lenDataSet$,
ylabel= Error on training task ,
 y label style={at={(0.1,0.5)}},
 yticklabels=\empty,
extra y ticks={0.05,.07, .1,.15},
extra y tick labels={ 0.05,.07, .1,.15},
yticklabels={,,},
grid=both,
xtick distance=1,
ytick distance=.05,
axis background/.style={fill=white},
legend style={at={(1,0.5)}, anchor= west, legend cell align=left, align=left, draw=none, legend columns=1}
]


\addplot [color=mycolor5,mark=triangle*,mark options={scale=.8}]
  table[row sep=crcr]{%
2	0.156\\
3	0.12\\
4	0.128\\
5	0.107\\
6	0.0969\\
7	0.0977\\
8	0.0997\\
9	0.0942\\
10	0.0909\\
};

\addplot [color=mycolor3,mark=o,mark options={scale=.8}]
  table[row sep=crcr]{%
2	0.13\\
3	0.109\\
4	0.0998\\
5	0.0913\\
6	0.0925\\
7	0.0905\\
8	0.087\\
9	0.0868\\
10	0.0854\\
};

\addplot [color=mycolor1,mark=square,mark options={scale=.8}]
  table[row sep=crcr]{%
2	 0.0880 \\
3	0.0810\\
4	 0.0765\\
5	0.0730 \\
6	0.0728\\
7	 0.0769\\
8	0.0766\\
9	0.0768\\
10	 0.0754\\
};

\addplot [color=mycolor6,dashed,mark=*,mark options={scale=.8}]
  table[row sep=crcr]{%
2	0.0786\\
3	0.0677\\
4	0.0763\\
5	0.0755\\
6	0.0672\\
7	0.069\\
8	0.0677\\
9	0.0694\\
10	0.0665\\
};

\addplot [color=mycolor2,dashed,mark=triangle*]
  table[row sep=crcr]{%
2	0.0628\\
3	0.0607\\
4	0.0567\\
5	0.0599\\
6	0.0579\\
7	0.0584\\
8	0.0589\\
9	0.0548\\
10	0.0575\\
};


\end{axis}
\end{tikzpicture}%
\caption{
Meta-training error} 
\label{pr_lab2}
\end{subfigure}\hfil
\begin{subfigure}[t]{.62\textwidth}
%
%
\definecolor{mycolor1}{rgb}{0.00000,0.44700,0.74100}%
\definecolor{mycolor2}{rgb}{0.85000,0.32500,0.09800}%
\definecolor{mycolor3}{rgb}{0.92900,0.69400,0.12500}%
\definecolor{mycolor4}{rgb}{0.49400,0.18400,0.55600}%
\definecolor{mycolor5}{rgb}{0.46600,0.67400,0.18800}%
\definecolor{mycolor6}{rgb}{0.30100,0.74500,0.93300}%
\definecolor{mycolor7}{rgb}{0.63500,0.07800,0.18400}%

\begin{tikzpicture}

\begin{axis}[%
width=2in,
height=1.4in,
at={(0.758in,0.481in)},
scale only axis,
grid=both,
xtick distance=1,
ytick distance=.1,
xmin=2,
xmax=10,
ymin=0.1,
ymax=0.55,
xlabel=Number of training tasks $\lenDataSet$,
ylabel= Error on new task ,
 yticklabels=\empty,
extra y ticks={0.1 ,.15,.3,.4,.5},
extra y tick labels={0.1 ,0.15,0.3,0.4,0.5},
yticklabels={,,},
 y label style={at={(.1,0.5)}},
axis background/.style={fill=white},
legend style={at={(1,0.5)}, anchor= west, legend cell align=left, align=left, draw=none, legend columns=1}
]

\addplot [color=mycolor5,mark=triangle*,mark options={scale=.8}]
  table[row sep=crcr]{%
2	0.499\\
3	0.497\\
4	0.463\\
5	0.47\\
6	0.46\\
7	0.54\\
8	0.505\\
9	0.49\\
10	0.459\\
};
\addlegendentry{$\lambda$-bound \cite{liu2021}}

\addplot [color=mycolor3,mark=o,mark options={scale=.8}]
  table[row sep=crcr]{%
2	0.463\\
3	0.458\\
4	0.488\\
5	0.482\\
6	0.503\\
7	0.468\\
8	0.476\\
9	0.475\\
10	0.443\\
};
\addlegendentry{Our fast-rate bound \eqref{fast_rate}}


 \addplot [color=mycolor1,mark=square,mark options={scale=.8}]
  table[row sep=crcr]{%
2	0.378\\
3	0.346\\
4	0.406\\
5	0.392\\
6	0.383\\
7	0.352\\
8	0.38\\
9	0.37\\
10	0.33\\
};
\addlegendentry{\cite{amit2018meta}}
 
\addplot [color=mycolor6,dashed,mark=*,mark options={scale=.8}]
  table[row sep=crcr]{%
2	0.3739  \\
3	0.3769\\
4	0.401\\
5	0.4065\\
6	0.3912\\
7	0.332\\
8	0.2991\\
9	 0.3241\\
10	 0.3265\\
};
\addlegendentry{Bound 1 \cite{mys}}
 \addplot [color=mycolor2,dashed,mark=triangle*]
  table[row sep=crcr]{%
2	0.203\\
3	0.175\\
4	0.18\\
5	0.199\\
6	0.184\\
7	0.18\\
8	 0.202\\
9	0.211\\
10	0.18\\
};
\addlegendentry{Our classic bound \eqref{New_McAll}}


\end{axis}
\end{tikzpicture}%
%
\caption{
Meta-test error} 
\label{pr_lab}
\end{subfigure}
 \caption{The average training and test errors versus the number of training-tasks. The number of training examples for each
task is  $600$ images, and during the meta-test phase, each task is constructed with  $100$ images. The number of epochs is $100$.   }
\label{figabc}
\squeezeup
\end{figure*}

 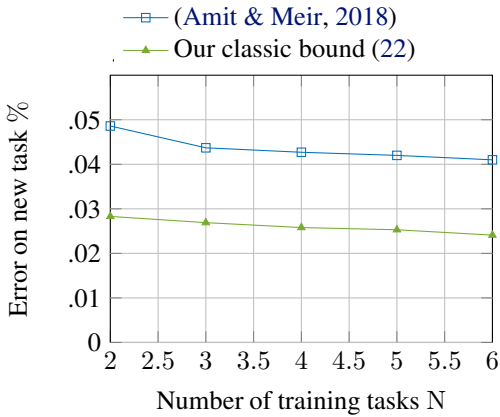
\begin{figure}[!t]
\centering
%
%
\definecolor{mycolor1}{rgb}{0.00000,0.44700,0.74100}%
\definecolor{mycolor2}{rgb}{0.85000,0.32500,0.09800}%
\definecolor{mycolor3}{rgb}{0.92900,0.69400,0.12500}%
\definecolor{mycolor4}{rgb}{0.49400,0.18400,0.55600}%
\definecolor{mycolor5}{rgb}{0.46600,0.67400,0.18800}%
\definecolor{mycolor6}{rgb}{0.30100,0.74500,0.93300}%
\definecolor{mycolor7}{rgb}{0.63500,0.07800,0.18400}%
\begin{tikzpicture}

\begin{axis}[%
width=2in,
height=1.4in,
at={(0.758in,0.481in)},
scale only axis,
xmin=2,
xmax=6,
ymin=0,
ymax=0.06,
yticklabels=\empty,
grid=both,
xtick distance=.5,
ytick distance=.5,
extra y ticks={0,0.01,.02, .03,.04,.05},
extra y tick labels={0,.01,.02, .03,.04,.05},
xlabel=Number of training tasks $\lenDataSet$,
ylabel= Error on new task $\%$,
axis background/.style={fill=white},
legend style={at={(.01,1.01)}, anchor=south west, legend cell align=left, align=left, draw=none, legend columns=1}
]
\addplot [color=mycolor1,mark=square,mark options={scale=.8}]
  table[row sep=crcr]{%
2	0.0486\\
3	0.0437\\
4	 0.0427\\
5	 0.042\\
6	0.041\\
};
\addlegendentry{\cite{amit2018meta} }

\addplot [color=mycolor5,mark=triangle*,mark options={scale=.8}]
  table[row sep=crcr]{%
2	0.0283\\
3	 0.0269 \\
4	0.0258\\
5	 0.0253\\
6	0.0241\\
};
\addlegendentry{Our classic bound \eqref{New_McAll}}

\end{axis}
\end{tikzpicture}%
\caption{
 The average test error of learning a new task
versus $\lenDataSet$.
 The number of training example of each task   $8000$. The number of epochs is $400$. } 
\label{pr_lab3}
\squeezeup
\end{figure}

Using the same experiment given by Section 5 of  \cite{amit2018meta} and also \cite{liu2021,mys}, we
compare our bounds with previous works. We reproduce
the experimental results of our method by directly running
the online code\footnote{ https://github.com/ron-amit/meta-learning-adjusting-priors2
} from \cite{amit2018meta}, and run our
algorithm by replacing others’ bounds with our bounds.

In image classification, the data samples $z=(x,y)$, consist of a an image, $x$ and a label $y$. We consider an experiment
based on augmentations of the MNIST dataset. We study
two experiments, namely permuted labels and permuted
pixels. For permuted labels, each task is created by a random permutation of image labels. For permuted pixels each task is created by a permutation of image pixels. The
pixel permutations are achieved by $100$ location swaps  to ensure
the task relatedness.

The network architecture used for the permuted-label
 experiment is a small CNN with two convolutions layers, a linear hidden layer and a linear output layer  \cite{amit2018meta}.
With  a 
learning rate of $10^{-3}$, we use the
hyper-prior, prior, hyper-posterior and posterior distributions given by \eqref{gua_hyper_prior}, \eqref{gua_prior}, \eqref{gua_hyper_posterior} and \eqref{gua_posterior}, respectively. 
We set $\kprior=100$, $\kposter=0.001$, and $\delta=0.1$. For each task $\tau_i$, and $k=1,\dots,d$, the posterior parameter $\log (\sigma_i^2(k))$ initialized by $\mathcal N(-10,0.01)$,
$\mu_i(k)$
is initialized randomly with the
Glorot method  \cite{Glorot2010}. Then, for different
bounds, by using backpropagation, we evaluate the gradient of the bound with respect to  $\bm \mu_i=(\mu_i(1),\dots,\mu_i(d))$. Then, we set $\mu_i(k)$ and $\sigma_i(k)$ as the means and variance
of $k$-th weight.
 The parameters 
$\mu_u(k)$ and $\sigma_u(k)$ are  similar in
structure, and the parameter $\theta$ is the vector containing
the weights of $\mathrm N$ tasks \cite{amit2018meta}.

Table \ref{table_result} shows the comparison of different PAC-Bayes
bounds for both permuted pixels and labels experiments.
The performance of our classic bounds is significantly  better than the existing bounds. Our fast-rate bound achieves
competitive performance on novel tasks. For permuted
labels, Figure \ref{pr_lab2} 
compares the average training error, and
Figure \ref{pr_lab} shows the test error of learning a new task for
different bounds. As shown in Figure \ref{pr_lab2}, the training error
of our classic bound  \eqref{New_McAll}  is comparable with other bounds.
However, in Figure \ref{pr_lab}, for new tasks, the performance of
our bound is much better than other bounds. Figure  \ref{pr_lab3}  compares the test error when the larger number of training examples is available. Again, our classic bound has
better performance.

\subsection{Conclusion}
\label{dis}
In this paper, for meta-learning setup, we have derived
a general PAC-Bayes bound which can recover existing known bounds and proposes new bounds. Based on our
extended PAC-Bayes bound, we have obtained a bound
from the fast-rate family and also a bound from the classic family.
The fast-rate bound yields to competitive experimental
results on novel tasks with respect to existing methods.
Unlike existing bound, to obtain the classic bound, we
used only one Markov’s inequality and by lower bounding the sum of environment-level and task-level convex
functions, we end up with a new classic bound. Practical
examples show that the new obtained classic bound reduces the meta overfitting problem. The main property
of the new classic bound is that it is expressed in terms of
one square. Thus, minimizing the new PAC-Bayes bound
leads to a simpler optimization problem, i.e., minimizing
an objective function which is linear with respect to KL-
divergences of posterior and prior distributions. We guess
that due to this property, the new proposed bound has
better performance on the meta-test set.

Potentially, our general PAC-Bayes bound holds for both
bounded and unbounded loss functions, as well as data-
dependent or data-free prior distributions. Here, we only
focused on data-free priors and bounded loss functions.
Generalizing to other scenarios is left to future work.

\section*{Acknowledgements}
The author is indebted to Prof. Giuseppe Durisi for his intellectual
leadership, many helpful suggestions and invaluable criticism.

This work has been funded by the European Union’s Horizon 2020
research and innovation programme under the Marie Sklodowska-Curie grant
agreement No. 893082.

\nocite{langley00}

\bibliography{example_paper}
\bibliographystyle{icml2022}

\newpage
\appendix
\onecolumn

\section{Proof of Theorem \ref{prop-gen-pac-tsk_two_KL}}
\label{Proof_prop-gen-pac-tsk_two_KL}
To bound the meta-generalization gap, we bound the generalization gap at task and environment levels, separately. At the
task level, for the task $\task_i$, the base-learner uses a prior and the samples $\trainTskSimp$  to output a distribution over hypotheses.
Here, we consider the prior over hypothesis   $(\hyperprior,\PriorTsk)$ as a joint distribution of one hyper-prior  $\hyperprior$ and
the prior $\PriorTsk$ depends
on  the hyper-prior.  Note that the posterior over the hypothesis can be any distribution, particularly a tuple   $(\hyperpost,  \PWcondZimUsimp)$ where firstly the hyperparameter $U$ is sampled from the hyper-posterior $\hyperpost$, and then the model parameter $W$ is sampled from $\PWcondZimUsimp$. Considering this approach, for any
  $\theta_{\textup{tsk}}\geq 0$, we have
 \begin{align}
 \theta_{\textup{tsk}}&\funcTsk[\mathbbm E_{\hyperpost}\mathbbm E_{\PWcondZimUsimp}\left( \mathrm L_{\mathrm P_{Z|\task_i}}(W) \right)]{ \mathbbm E_{\hyperpost}\mathbbm E_{\PWcondZimUsimp}\left( \empLossTskRV\right)}  \nonumber\\
&
\leq
 \theta_{\textup{tsk}}\mathbbm E_{\hyperpost}\mathbbm E_{\PWcondZimUsimp}\left[\funcTsk[\mathrm L_{\mathrm P_{Z|\task_i}}(W) ]{ \empLossTskRV}\right]\label{pr_tsk_c}
\\ 
  & \leq \Dkl\left(\hyperpost\PWcondZimUsimp||
     \PriorEnvSimp \PriorTskSimp
     \right)  +\log\left(
     \mathbbm E_{   \PriorEnvSimp \PriorTskSimp} e^{\theta_{\textup{tsk}}\funcTsk[ \expLossTskRV]{\empLossTskRV}}
     \right)
     \label{2021.03.10_15.13}\\
     &=  \Dkl\left(\hyperpost||
     \PriorEnvSimp 
     \right) +\mathbbm E_{\hyperpost}\left[\Dkl\left(
     \PWcondZimUsimp||\PriorTskSimp\right)\right]
     +\log\left(
     \mathbbm E_{   \PriorEnvSimp \PriorTskSimp} e^{\theta_{\textup{tsk}}\funcTsk[ \expLossTskRV]{\empLossTskRV}}
     \right), \label{pr_tsk_d}
     \end{align}
     where 
     $   \decomMeta[u]{\bm Z^{\lenZm}_i}{\task_i}$ 
     and 
     $\empLossTsksimple[w]{}$ are defined in \eqref{decom-pac}
        \eqref{emp_loss_task}, respectively.
 Since $\funcTskRV(\cdot)$ is convex, in  \eqref{pr_tsk_c} we applied Jensen's inequality, and  \eqref{2021.03.10_15.13} follows from the Donsker-Varadhan
theorem  \eqref{Donsker}. Finally, \eqref{pr_tsk_d}
follows from the definition of the KL-divergence.

Next, we average both sides of \eqref{pr_tsk_d} over $\lenDataSet$ tasks. Recalling that 
 $\funcTsk[a]{b}$ is convex in both $a$ and $b$, we have $\funcTsk[\frac{1}{\lenDataSet}\sum_i^{\lenDataSet} a_i]{\frac{1}{\lenDataSet}\sum_i^{\lenDataSet} b_i}\leq \frac{1}{\lenDataSet}\sum_i^{\lenDataSet}\funcTsk[a_i]{b_i}$. By applying this fact, in view of \eqref{decom-pac_tot} and \eqref{sharu_eq11},  using  $\log\prod_i a_i=\sum_i \log a_i$,  we  find that 
      \begin{align}
\theta_{\textup{tsk}}\funcTsk[\mathbbm E_{U\sim\hyperpost}\left( \decomMetaRVTot\right)]{\mathbbm E_{U\sim\hyperpost} \empLossEnvRV}
      \leq
          \Dkl\left(\hyperpost||
     \PriorEnvSimp
     \right)+ \frac{1}{\lenDataSet} \mathbbm E_{\hyperpost}\left(\sum_{i=1}^{\lenDataSet}\Dkl\left(\PWcondZimUsimp||
     \PriorTsk
     \right)\right)
          \nonumber\\
     +\frac{1}{\lenDataSet} \log\underbrace{\left(
     \prod_{i=1}^{\lenDataSet}\mathbbm E_{ \JointPriorTskSimp} e^{ \theta_{\textup{tsk}} \funcTsk[
     \expLossTskRV]{
     \empLossTskRV}}
     \right)}_{\Upsilon_{\textup{tsk}}}.\label{2021.03.10_15.14}
 \end{align}

Similarly, at the environment level, by setting hyper-prior and hyper-posterior as  $\hyperprior$ and $\hyperpost$, respectively,  using Jensen's inequality, and applying the  Donsker-Varadhan
theorem \eqref{Donsker}, for $\theta_{\textup{env}}\geq 0$
we have
\begin{align}
\theta_{\textup{env}}\funcEnv[\mathbbm E_{U\sim\hyperpost}\left(
\expLossEnvRV
\right)]{\mathbbm E_{U\sim\hyperpost}\left( \decomMetaRVTot\right)}\hspace{15em}\nonumber\\
\leq
\Dkl\left(\hyperpost||
     \PriorEnvSimp
     \right)+\log\underbrace{\left(\mathbbm E_{\PriorEnvSimp}e^{\theta_{\textup{env}}\funcEnv[
\expLossEnvRV
]{\decomMetaRVTot}}\right)}_{\Upsilon_{\textup{env}}}.\label{2021.03.10_15.15}
\end{align}

Now, dividing both sides of
\eqref{2021.03.10_15.14} (resp. \eqref{2021.03.10_15.15}) by $\theta_{\textup{tsk}}$ (resp. $\theta_{\textup{env}}$), 
summing up both sides of the obtained inequalities, and
using the fact that  $\log(a)+\log(b)=\log(a.b)$, we finally obtain
\begin{align}
&\funcEnv[\mathbbm E_{U\sim\hyperpost}\left(
\expLossEnvRV
\right)]{\mathbbm E_{U\sim\hyperpost}\left( \decomMetaRVTot\right)}
+\funcTsk[\mathbbm E_{U\sim\hyperpost}\left( \decomMetaRVTot\right)]{\mathbbm E_{U\sim\hyperpost} \empLossEnvRV}\nonumber\\
&\hspace{6em}\leq \frac{1}{\lenDataSet\cdot\theta_{\textup{tsk}}}
     \mathbbm E_{\hyperpost}\left(\sum_{i=1}^{\lenDataSet}\Dkl\left(\PWcondZimUsimp||
     \PriorTsk
     \right)\right)
+\left( \frac{1}{\theta_{\textup{tsk}}}+\frac{1}{\theta_{\textup{env}}}\right)\Dkl\left(\hyperpost||
     \PriorEnvSimp
     \right)+ \log\left(\Upsilon_{\textup{tsk}}^{\frac{1}{\lenDataSet\cdot\theta_{\textup{tsk}}}}\cdot
     \Upsilon_{\textup{env}}^{\frac{1}{\theta_{\textup{env}}}}
\right).\label{both_tsk_env}
\end{align}
Finally, by applying the Markov’s inequality, i.e., $\mathbbm P[\Upsilon\geq  \mathbbm E[\Upsilon]/\delta]\leq \delta$  to the  $\Upsilon_{\textup{tsk}}^{\frac{1}{\theta_{\lenDataSet\cdot\textup{tsk}}}}\cdot
     \Upsilon_{\textup{env}}^{\frac{1}{\theta_{\textup{env}}}}$ term, from \eqref{both_tsk_env}, we conclude that with probability at least $1-\delta$
\begin{align}
  \funcEnv[\mathbbm E_{U\sim\hyperpost}\left(
\expLossEnvRV
\right)]{\mathbbm E_{U\sim\hyperpost}\left( \decomMetaRVTotRV\right)}
+\funcTsk[\mathbbm E_{U\sim\hyperpost}\left( \decomMetaRVTotRV\right)]{\mathbbm E_{U\sim\hyperpost} \empLossEnvRV}\nonumber\\  
\leq
  \frac{1}{\lenDataSet\cdot\theta_{\textup{tsk}}}
     \mathbbm E_{\hyperpost}\left(\sum_{i=1}^{\lenDataSet}\Dkl\left(\PWcondZimUsimp||
     \PriorTsk
     \right)\right)
  + \left( \frac{1}{\theta_{\textup{tsk}}}+\frac{1}{\theta_{\textup{env}}}\right)\Dkl\left(\hyperpost||
     \PriorEnvSimp
     \right)+\log\frac{\mathbbm E_{\mathrm P_{T_{1:\lenDataSet}}}\mathbbm E_{\mathrm P_{\trainEnvSimp|T_{1:\lenDataSet}}}\left(\Upsilon_{\textup{tsk}}^{\frac{1}{\lenDataSet\cdot\theta_{\textup{tsk}}}}\cdot
     \Upsilon_{\textup{env}}^{\frac{1}{\theta_{\textup{env}}}}
\right)}{\delta},
\end{align}
which proves \eqref{cor_tup_1}.

Next, to prove \eqref{cor_tup_2},   we apply Markov’s inequality to both \eqref{pr_tsk_d} and \eqref{2021.03.10_15.15}. From \eqref{pr_tsk_d}, we find that with probability at least $1-\delta_i$ under distribution $\mathrm P_{T_{1:\lenDataSet}} \mathrm P_{\trainEnvSimp|T_{1:\lenDataSet}}$, we have
 \begin{align}
 \theta_{\textup{tsk}}&\funcTsk[\mathbbm E_{\hyperpost}\mathbbm E_{\PWcondZimUsimp}\left( \mathrm L_{\mathrm P_{Z|\taskRV_i}}(W)  \right)]{ \mathbbm E_{\hyperpost}\mathbbm E_{\PWcondZimUsimp}\left( \empLossTskRV\right)}  \nonumber\\
     &\leq  \Dkl\left(\hyperpost||
     \PriorEnvSimp 
     \right) +\mathbbm E_{\hyperpost}\left[\Dkl\left(
     \PWcondZimUsimp||\PriorTskSimp\right)\right]
     +\log\left(
\frac{  \mathbbm E_{\mathrm P_{T_{1:\lenDataSet}}}\mathbbm E_{\mathrm P_{\trainEnvSimp|T_{1:\lenDataSet}}}  \mathbbm E_{   \PriorEnvSimp \PriorTskSimp} e^{\theta_{\textup{tsk}}\funcTsk[ \expLossTskRVRV]{\empLossTskRV}}}{\delta_i}
     \right). \label{pr_tsk_e2}
     \end{align}
Recalling   that  from $\funcTsk[a]{b}\leq c_{\textup{tsk}}$ , we can conclude $a\leq k_t\cdot b+ \funcinvTsk{c_{\textup{tsk}}}$, from \eqref{pr_tsk_e2}, by dividing both sides of $a\leq k_t\cdot b+ \funcinvTsk{c_{\textup{tsk}}}$ by $\lenDataSet$, with probability at least $1-\delta_i$,
we have
 \begin{align}
&\frac{1}{\lenDataSet}\mathbbm E_{\hyperpost}\mathbbm E_{\PWcondZimUsimp}\left( \mathrm L_{\mathrm P_{Z|\taskRV_i}}(W)  \right)\leq \frac{k_t}{\lenDataSet} \mathbbm E_{\hyperpost}\mathbbm E_{\PWcondZimUsimp}\left( \empLossTskRV\right)+\frac{1}{\lenDataSet}\funcinvTsk{\mathrm B_t}
, \label{pr_tsk_e}
     \end{align}
     where
     \begin{align}
         \mathrm B_t=\frac{1}{ \theta_{\textup{tsk}}}
         \Dkl\left(\hyperpost||
     \PriorEnvSimp 
     \right) +\frac{1}{ \theta_{\textup{tsk}}}\mathbbm E_{\hyperpost}\left[\Dkl\left(
     \PWcondZimUsimp||\PriorTskSimp\right)\right]
     +\frac{1}{ \theta_{\textup{tsk}}}\log\left(
\frac{  \mathbbm E_{\mathrm P_{T_{1:\lenDataSet}}}\mathbbm E_{\mathrm P_{\trainEnvSimp|T_{1:\lenDataSet}}}  \mathbbm E_{   \PriorEnvSimp \PriorTskSimp} e^{\theta_{\textup{tsk}}\funcTsk[ \expLossTskRVRV]{\empLossTskRV}}}{\delta_i}
     \right).
     \end{align}
Here,
in Lemma \ref{Lemma_sum}, we set  $f_i$ as the left hand side of \eqref{pr_tsk_e} and  $a_i$ as the right hand side of \eqref{pr_tsk_e}. Thus, from 
Lemma \ref{Lemma_sum},
we conclude that with probability at least $1-\sum_i\delta_i$,
 \begin{align}
\mathbbm E_{U\sim\hyperpost}\left[\frac{1}{\lenDataSet}\sum_{i=1}^\lenDataSet\mathbbm E_{W\sim\PWcondZimUsimp}\left( \mathrm L_{\mathrm P_{Z|\taskRV_i}}(W)  \right)\right]
\leq \frac{k_t}{\lenDataSet}\sum_{i=1}^\lenDataSet \mathbbm E_{\hyperpost}\mathbbm E_{\PWcondZimUsimp}\left( \empLossTskRV\right)+\frac{1}{\lenDataSet}\sum_{i=1}^\lenDataSet\funcinvTsk{\mathrm B_t}. \label{pr_tsk_f}
     \end{align}
Finally, in view of \eqref{decom-pac_tot}  and \eqref{sharu_eq11},     \eqref{pr_tsk_f} can be written as
 \begin{align}
\mathbbm E_{\hyperpost}\left[ \decomMetaRVTotRV\right] \leq k_t\cdot\mathbbm E_{\hyperpost}\left[\empLossEnv[U] \right]+\frac{1}{\lenDataSet}\sum_{i=1}^\lenDataSet\funcinvTsk{\mathrm B_t}. \label{2021.03.15_17.55_b}
     \end{align}

Similarly, at the environment level from $\funcEnv[a]{b}\leq c_{\textup{env}}$, we can conclude $a\leq k_e\cdot b+ \funcinvEnv{c_{\textup{env}}}$. Considering this fact,
by applying the Markov’s inequality to  \eqref{2021.03.10_15.15},
with probability at least $1-\delta_0$,
we have
\begin{align}
\mathbbm E_{\hyperpost}\big(\expLossEnvRV\big)\leq k_e\cdot \mathbbm E_{\hyperpost}\Big(\decomMetaRVTotRV\Big)+
\funcinvEnv{
\mathrm B_e},\label{2021.03.15_18.24}
\end{align}
where
\begin{align}
    \mathrm B_e=\frac{1}{\theta_{\textup{env}}}\Dkl\left(\hyperpost||
     \PriorEnvSimp
     \right)+\frac{1}{\theta_{\textup{env}}}\log \left(\frac{\mathbbm E_{\mathrm P_{T_{1:\lenDataSet}}}\mathbbm E_{\mathrm P_{\trainEnvSimp|T_{1:\lenDataSet}}}\mathbbm E_{\PriorEnvSimp}e^{\theta_{\textup{env}}\funcEnv[
\expLossEnvRV
]{\decomMetaRVTotRV}}}{\delta_0}\right). \label{2021.06.14_10.40}
\end{align}
Here, again we use Lemma \ref{Lemma_sum}. 
In  Lemma \ref{Lemma_sum}, we set $\lenDataSet=2$, $f_1$ and $a_1$  as the $k_e\geq 0$ times of the left and right hands  side of \eqref{2021.03.15_17.55_b}, respectively, and also $f_2$ and $a_2$  as the   left and right hands  side of \eqref{2021.03.15_18.24}, respectively. Thus,    with probability at least $1-\sum_i\delta_i-\delta_0$,
\begin{align}
\mathbbm E_{\hyperpost}\left[\expLossEnvRV\right]\leq k_e\cdot k_t\cdot\mathbbm E_{\hyperpost}\left[\empLossEnv[U]\right]+
\funcinvEnv{
\mathrm B_e}%
+ \frac{k_e}{\lenDataSet}\sum_{i=1}^\lenDataSet\funcinvTsk{\mathrm B_t}
.\label{2021.03.15_18.31}
\end{align}
Finally, setting  $\delta_0=\frac{\delta}{2}$, $\delta_i=\frac{\delta}{2\lenDataSet}$ in \eqref{2021.03.15_18.31}, we conclude \eqref{cor_tup_2}.

 \section{Re-obtaining the known PAC-Bayes Bounds}
 \label{derv_previous}
 In this section, we present the derivation of bounds, summarized in Table \ref{table_reobtain}.
Firstly,  to obtain Theorem 2 of \cite{amit2018meta}, in \eqref{cor_tup_2},
we set  $\theta_{\textup{tsk}}=\theta_{\textup{env}}=1$, and 
$\funcTsk[a]{b}=2(\lenZm-1)(a-b)^2$, $\funcEnv[a]{b}=2(\lenDataSet-1)(a-b)^2$. These choices lead to $k_e=k_t=1$,  $\funcinvTsk{c}=\sqrt{\frac{c}{2(\lenZm-1)}}$ and also  $\funcinvEnv{c}=\sqrt{\frac{c}{2(\lenDataSet-1)}}$. To simplify 
$\Btsk$ and 
$\Benv$ given by \eqref{B_task_eq} and \eqref{B_env_eq}, we use Lemma \ref{prior-indep}.  Since the prior in independent of the data, by interchanging the order of expectations over $\mathrm P_{T_{1:\lenDataSet}}\mathrm P_{ \trainEnvSimp|T_{1:\lenDataSet}}$ and priors, in view of \eqref{B_task_eq} and \eqref{B_env_eq}, we find that
\begin{align}
  & 
        \mathbbm E_{ \PriorEnvSimp \PriorTsk}  \mathbbm E_{\mathrm P_{T_{1:\lenDataSet}}\mathrm P_{ \trainEnvSimp|T_{1:\lenDataSet}}}\left[  e^{
2(\lenZm-1)\left(
 \expLossTskRV-  \empLossTskRV\right)^2}\right]  \leq \lenZm,\label{2021.03.21_08.26}\\
  &  
    \mathbbm E_{\PriorEnvSimp}  \mathbbm E_{\mathrm P_{ \trainEnvSimp}}\left[  e^{
2(\lenDataSet-1)\left(
\expLossEnvRV
-\decomMetaRVTot\right)^2}\right]  \leq \lenDataSet,\label{2021.03.21_08.25}
\end{align}
where for \eqref{2021.03.21_08.26} (res. \eqref{2021.03.21_08.25}),  we used Lemma \ref{prior-indep} by setting  $\lambda=\frac{2(\lenZm-1)}{\lenZm}$ (resp. $\lambda=\frac{2(\lenDataSet-1)}{\lenDataSet}$) and $\sigma=0.5$. We recall that 
since in \cite{amit2018meta} the loss function is bounded on $[0,1]$, we set $\sigma=0.5$ in Lemma \ref{prior-indep}. Now, inserting \eqref{2021.03.21_08.26}  and  \eqref{2021.03.21_08.25} into \eqref{B_task_eq} and \eqref{B_env_eq}, from \eqref{cor_tup_2} we conclude that
       \begin{align}
\mathbbm E_{U\sim\hyperpost}\left(
\expLossEnvRV
\right)-\mathbbm E_{U\sim\hyperpost} \empLossEnvRV
       \leq      
     \sqrt{\frac{\Dkl\left(\hyperpost||
     \PriorEnvSimp
     \right)+
     \log\left(\frac{2
\lenDataSet}{\delta}\right)
    }{2(\lenDataSet-1)} } 
     \hspace{0em}\nonumber\\
      +\frac{1}{\lenDataSet}\sum_{i=1}^{\lenDataSet}\sqrt{\frac{
    \Dkl\left(\hyperpost||
     \PriorEnvSimp 
     \right) + \mathbbm E_{\hyperpost}\left[\Dkl\left(
     \PWcondZimUsimp||\PriorTskSimp\right)\right]
     +\log
      \frac{
2\lenDataSet\lenZm
}
{\delta}}{2(\lenZm-1)}},
 \end{align}
 which is the same as the bound presented in Theorem 2 of \cite{amit2018meta}.

Next, to obtain Theorem 2 of \cite{rothfuss2020pacoh}, in \eqref{cor_tup_1} we set  $\funcEnv[a]{b}=\funcTsk[a]{b}=(a-b)$,  $\theta_{\textup{tsk}}=\beta$, and $\theta_{\textup{env}}=\lambda$, and these choices leads to $k_t=k_e=1$ and $\funcinvTsk{c}=\funcinvEnv{c}=c$. 
 To simplify 
the log-term, since the prior in independent of the data, by interchanging the order of expectations over $\mathrm P_{T_{1:\lenDataSet}}\mathrm P_{ \trainEnvSimp|T_{1:\lenDataSet}}$ and priors, and 
recalling that the loss function is bounded
 on the interval $[0,1]$, and hence it  is sub-Gaussian with parameter $\sigma=0.5$, we can conclude
\begin{align}
\mathbbm E_{U\sim\hyperpost}\left(
\expLossEnvRV
\right)-\mathbbm E_{U\sim\hyperpost} \empLossEnvRV
 \leq \min_{\beta,\lambda\geq 0}
(\frac{\lambda}{8\lenDataSet}+\frac{\lambda}{8\lenZm})-\frac{1}{\sqrt\lenDataSet}\log\delta
\nonumber\\
+ \frac{1}{\beta}
     \mathbbm E_{\hyperpost}\left(\sum_{i=1}^{\lenDataSet}\Dkl\left(\PWcondZimUsimp||
     \PriorTsk
     \right)\right) 
    + \left( \frac{1}{\beta}+\frac{1}{\lambda}\right)\Dkl\left(\hyperpost||
     \PriorEnvSimp
     \right),
\end{align}
which is the same as Theorem 2 of \cite{rothfuss2020pacoh}.

Next, we re-obtain Theorem 1 of \cite{liu2021}.
Firstly, in \eqref{cor_tup_2},
we set $\theta_{\textup{tsk}}=\theta_{\textup{env}}=1$, and  $\funcTsk[a]{b}=\lenZm kl(a,b)$ and $\funcEnv[a]{b}=2(\lenDataSet-1)(a-b)^2$.
These choices lead to $k_e=1$ $\funcinvEnv{c}=\sqrt{\frac{c}{2(\lenDataSet-1)}}$, and 
using further relaxation,
from $\lenZm kl(a,b)\leq c_{\textup{tsk}}$, it can be proved that
$a\leq b/(1-0.5 \lambda)+c_{\textup{tsk}}/(\lenZm.\lambda(1-0.5\lambda))$ for $\lambda\in(0,2)$
\cite{Thiemann2017}. Thus, we have $k_t$
and $\funcinvTsk{c}= c/(\lenZm\lambda(1-0.5\lambda))$. It remains to obtain the log-terms of 
$\Btsk$ and 
$\Benv$ given by \eqref{B_task_eq} and \eqref{B_env_eq}.
Again, assuming the prior in independent of the data, by interchanging the order of expectations over $\mathrm P_{T_{1:\lenDataSet}}\mathrm P_{ \trainEnvSimp|T_{1:\lenDataSet}}$ and priors, we find \eqref{2021.03.21_08.25}, and
\begin{align}
  \mathbbm E_{   \PriorEnvSimp \PriorTskSimp}  \mathbbm E_{\mathrm P_{T_{1:\lenDataSet}}}\mathbbm E_{\mathrm P_{\trainEnvSimp|T_{1:\lenDataSet}}}   e^{\lenZm kl\left(\mathrm L_{\mathrm P_{Z|\taskRV_i}}(W),\empLossTskRV \right)}\leq 2\sqrt{\lenZm},\label{2021.01.26_12.22}
\end{align}
where in \eqref{2021.01.26_12.22}, we used Lemma \ref{kl_lem}.
Applying all these facts to \eqref{cor_tup_2}, we obtain
\begin{align}
    \mathbbm E_{U\sim\hyperpost}\left[\expLossEnvRV\right]\leq \frac{1}{(1-0.5 \lambda)}\mathbbm E_{U\sim\hyperpost}\left[\empLossEnv[U]\right]+\sqrt{\frac{\Dkl\left(\hyperpost||
     \PriorEnvSimp
     \right)+\log\frac{2\lenDataSet}{\delta}}{2(\lenDataSet-1)}}\nonumber\\
     +\frac{1}{\lenDataSet}\sum_{i=1}^\lenDataSet\frac{ \Dkl\left(\hyperpost||
     \PriorEnvSimp 
     \right) +\mathbbm E_{\hyperpost}\left[\Dkl\left(
     \PWcondZimUsimp||\PriorTskSimp\right)\right]
     +
     \log\frac{4\lenDataSet\sqrt{\lenZm}}{\delta}
     }{(\lenZm\lambda(1-0.5\lambda))},
\end{align}
which is the same as Theorem 1 of \cite{liu2021}.

Similar approach can be applied to the  bounds presented in \cite{mys}. For the three  bounds considered in \cite{mys}, KL-divergence is chosen for both task level and environment level. To bound the log-terms of \eqref{B_task_eq} and \eqref{B_env_eq}, we need to use Lemma 2 of  \cite{mys}. For the affine transformation steps, at the environment level, we  use Pinsker's inequality.
\section{Proof Theorem \ref{fast-rate-bnd1}}
\label{Proof_fast-rate-bnd1}
To prove  Theorem \ref{fast-rate-bnd1}, in \eqref{cor_tup_2}, we set 
  $\theta_{\textup{tsk}}=\theta_{\textup{env}}=1$, 
 $\funcEnv[a]{b}=\lenDataSet\Dgama(b,a)$ and $\funcTsk[a]{b}=\lenZm\Dgama(b,a)$. Using Lemma \ref{lemma2_MCallester}, from $\lenDataSet\Dgama(b,a) \leq c_e$, we conclude that for $\gamma\in(-2,0)$, 
 $a\leq b/(1+0.5\gamma)-c_{e}/(\lenDataSet.\gamma(1+0.5\gamma))$. In other words,
 $k_e=1/(1+0.5\gamma)$
 and  $\funcinvEnv{c}=\frac{-c}{\lenDataSet\gamma(1+0.5\gamma)}$ (similarly for the task level).  It remains to determine the log-terms appeared in $\Btsk$ and 
$\Benv$ given by \eqref{B_task_eq} and \eqref{B_env_eq}, respectively.  
   Since the prior is independent of the data, by interchanging the order of expectations over $\mathrm P_{T_{1:\lenDataSet}}\mathrm P_{ \trainEnvSimp|T_{1:\lenDataSet}}$ and priors, using  Lemma 
\ref{D_gama}, in view of \eqref{B_task_eq} and \eqref{B_env_eq}, we find that
\begin{align}
    &\mathbbm E_{\PriorEnvSimp}\mathbbm E_{\mathrm P_{T_{1:\lenDataSet}}}\mathbbm E_{\mathrm P_{\trainEnvSimp|T_{1:\lenDataSet}}}e^{\lenDataSet\Dgama\left(\decomMetaRVTotRV,\expLossEnvRV
    \right)}\leq 1,\\
    &  
   \mathbbm E_{   \PriorEnvSimp \PriorTskSimp} \mathbbm E_{\mathrm P_{T_{1:\lenDataSet}}}\mathbbm E_{\mathrm P_{\trainEnvSimp|T_{1:\lenDataSet}}}   e^{\lenZm\Dgama\left(\empLossTskRV
   ,\mathrm L_{\mathrm P_{Z|\taskRV_i}}(W)
   \right)
   }
    \leq 1.
\end{align}
Applying all these facts to \eqref{cor_tup_2}, we find that
\begin{align}
 \mathbbm E_{U\sim\hyperpost}\left[\expLossEnvRV\right]\leq \frac{1}{(1+0.5\gamma_e)}\cdot \frac{1}{(1+0.5\gamma_t)}\cdot\mathbbm E_{U\sim\hyperpost}\left[\empLossEnv[U]\right]   -\frac{\Dkl\left(\hyperpost||
     \PriorEnvSimp
     \right)+\log\frac{2}{\delta}}{\lenDataSet\gamma_e(1+0.5\gamma_e)}\nonumber\\
    -\frac{1}{\lenDataSet(1+0.5\gamma_e)}\sum_{i=1}^\lenDataSet \frac{
    \Dkl\left(\hyperpost||
     \PriorEnvSimp 
     \right) +
    \mathbbm E_{\hyperpost}\left[\Dkl\left(
     \PWcondZimUsimp||\PriorTskSimp\right)\right]+
    \log\frac{2\lenDataSet }{\delta}
    }{\lenZm\gamma_t(1+0.5\gamma_t)}.
\end{align}
Setting $\lambda_e=-1/\gamma_e$ and 
$\lambda_t=-1/\gamma_t$, we conclude the proof.

 \section{Proof of Theorem \ref{the_Nw_McAll}}
 \label{Proof_the_Nw_McAll}
 Setting  $\theta_{\textup{tsk}}=\theta_{\textup{env}}=1$,  $\funcTsk[a]{b}=2(\lenZm-1)(a-b)^2$ and $\funcEnv[a]{b}=(\lenDataSet-1)(a-b)^2$  \eqref{cor_tup_1}, leads to
 \begin{align}
&(\lenDataSet-1)\left(\mathbbm E_{U\sim\hyperpost}\left(
\expLossEnvRV
\right)-\mathbbm E_{U\sim\hyperpost}\left( \decomMetaRVTotRV\right)\right)^2\nonumber\\
&\hspace{5em}+2(\lenZm-1)\left(\mathbbm E_{U\sim\hyperpost}\big( \decomMetaRVTotRV\big)-\mathbbm E_{U\sim\hyperpost} \big(\empLossEnvRV\big)\right)^2\nonumber\\
& \leq 
\Dkl\left(\hyperpost||
     \PriorEnvSimp
     \right)+ \frac{1}{\lenDataSet}
     \mathbbm E_{\hyperpost}\left(\sum_{i=1}^{\lenDataSet}\Dkl\left(\PWcondZimUsimp||
     \PriorTskSimp
     \right)\right)  +  \log\frac{\mathbbm E_{\mathrm P_{T_{1:\lenDataSet}}}\mathbbm E_{\mathrm P_{\trainEnvSimp|T_{1:\lenDataSet}}}\left(\Upsilon_{\textup{tsk}}^{\frac{1}{\lenDataSet}}\cdot
     \Upsilon_{\textup{env}}
\right)}{\delta}.
\end{align}
Then, following the same steps to obtain \eqref{inv_ineq}, we can show that
\begin{align}
 \frac{2(\lenZm-1)(\lenDataSet-1)}{2(\lenZm-1)+(\lenDataSet-1)}\left( \mathbbm E_{U\sim\hyperpost}\left(
\expLossEnvRV
\right)- \mathbbm E_{U\sim\hyperpost} \big(\empLossEnvRV\big)  \right)^2  \hspace{10em} \nonumber\\\leq (\lenDataSet-1)\left(\mathbbm E_{U\sim\hyperpost}\left(
\expLossEnvRV
\right)-\mathbbm E_{U\sim\hyperpost}\left( \decomMetaRVTotRV\right)\right)^2\hspace{2em}\nonumber\\
+2(\lenZm-1)\left(\mathbbm E_{\hyperpost}\big( \decomMetaRVTotRV\big)-\mathbbm E_{\hyperpost} \big(\empLossEnvRV\big)\right)^2,
\end{align}
and hence
 \begin{align}
 &\frac{2(\lenZm-1)(\lenDataSet-1)}{2(\lenZm-1)+(\lenDataSet-1)}\left( \mathbbm E_{U\sim\hyperpost}\left(
\expLossEnvRV
\right)- \mathbbm E_{U\sim\hyperpost} \big(\empLossEnvRV\big)  \right)^2 \nonumber\\
& \leq 
\Dkl\left(\hyperpost||
     \PriorEnvSimp
     \right)+ \frac{1}{\lenDataSet}
     \mathbbm E_{\hyperpost}\left(\sum_{i=1}^{\lenDataSet}\Dkl\left(\PWcondZimUsimp||
     \PriorTskSimp
     \right)\right)  +  \log\frac{\mathbbm E_{\mathrm P_{T_{1:\lenDataSet}}}\mathbbm E_{\mathrm P_{\trainEnvSimp|T_{1:\lenDataSet}}}\left(\Upsilon_{\textup{tsk}}^{\frac{1}{\lenDataSet}}\cdot
     \Upsilon_{\textup{env}}
\right)}{\delta}.\label{cor_tup_1_new_mc}
\end{align}

 Now, the log-term appeared in \eqref{cor_tup_1_new_mc}
can be bounded  as
 \begin{align}
 \mathbbm E_{\mathrm P_{\taskRV_{1:\lenDataSet}}} \mathbbm E_{\mathrm P_{\trainEnvSimp}|\mathrm P_{\taskRV_1:\taskRV_\lenDataSet}}\left(\Upsilon_{\textup{tsk}}^{\frac{1}{\lenDataSet}}\cdot
     \Upsilon_{\textup{env}}
\right)   
\leq 
 \sqrt{\mathbbm E_{\mathrm P_{\taskRV_{1:\lenDataSet}}} \mathbbm E_{\mathrm P_{\trainEnvSimp}|\mathrm P_{\taskRV_1:\taskRV_\lenDataSet}}\left(\Upsilon_{\textup{tsk}}^{\frac{1}{\lenDataSet}}\right)^2\cdot
\mathbbm E_{\mathrm P_{\taskRV_{1:\lenDataSet}}} \mathbbm E_{\mathrm P_{\trainEnvSimp}|\mathrm P_{\taskRV_1:\taskRV_\lenDataSet}} \left(    \Upsilon_{\textup{env}}
\right)^2 }  \label{dif_1},
 \end{align}
 where in \eqref{dif_1}, we applied Cauchy-Schwartz inequality (or H\"{o}lder’s inequality). Next, by  setting  $\kappa=2$, $\theta_{\textup{tsk}}=\theta_{\textup{env}}=1$,  $\funcTsk[a]{b}=2(\lenZm-1)(a-b)^2$ and $\funcEnv[a]{b}=(\lenDataSet-1)(a-b)^2$, in \eqref{up_env} and \eqref{up_tsk}, we have
 \begin{align}
 \mathbbm E_{\mathrm P_{\taskRV_{1:\lenDataSet}}} \mathbbm E_{\mathrm P_{\trainEnvSimp}|\mathrm P_{\taskRV_1:\taskRV_\lenDataSet}}\left(\Upsilon_{\textup{tsk}}^{\frac{1}{\lenDataSet}}\right)^2&=
  \mathbbm E_{\mathrm P_{\taskRV_{1:\lenDataSet}}} \mathbbm E_{\mathrm P_{\trainEnvSimp}|\mathrm P_{\taskRV_1:\taskRV_\lenDataSet}}\left(
     \prod_{i=1}^{\lenDataSet}\mathbbm E_{ \JointPriorTskSimp} e^{ \theta_{\textup{tsk}}
     \funcTsk[
     \mathrm L_{\mathrm P_{Z|\taskRV_i}}(W)]{
     \empLossTskRV}}
  \right)^{\frac{2}{\lenDataSet}}\label{tsk_eq_1}\\
 & \leq 
\left(   \mathbbm E_{\mathrm P_{\taskRV_{1:\lenDataSet}}} \mathbbm E_{\mathrm P_{\trainEnvSimp}|\mathrm P_{\taskRV_1:\taskRV_\lenDataSet}}
     \prod_{i=1}^{\lenDataSet}\mathbbm E_{ \JointPriorTskSimp} e^{ 2(\lenZm-1)
     \left(
     \mathrm L_{\mathrm P_{Z|\taskRV_i}}(W)-
     \empLossTskRV\right)^2}
  \right)^{\frac{2}{\lenDataSet}}\label{tsk_eq_2}\\
  &=\left(  
     \prod_{i=1}^{\lenDataSet}\mathbbm E_{ \JointPriorTskSimp}\mathbbm E_{\mathrm P_{\taskRV_i \bm Z^\lenZm_i}} e^{ 2(\lenZm-1)
     \left(
     \mathrm L_{\mathrm P_{Z|\taskRV_i}}(W)-
     \empLossTskRV\right)^2}
  \right)^{\frac{2}{\lenDataSet}}\label{tsk_eq_3}\\
  &\leq \left(   \prod_{i=1}^{\lenDataSet} \lenZm \right)^{\frac{2}{\lenDataSet}}=
  \lenZm^2
  \label{tsk_eq_4}
 \end{align}
 where in \eqref{tsk_eq_1} we applied \eqref{up_tsk}. In \eqref{tsk_eq_2}, since $a^{2/\lenDataSet}$ is a concave function for $\lenDataSet\geq 2$, we used Jensen's inequality. Since, tasks are assumed to be independent, and 
the prior is independent of the data, by interchanging the order of expectations over $\mathrm P_{\taskRV\bm Z_i^\lenZm}$ and $\JointPriorTskSimp$, we obtained \eqref{tsk_eq_3}. 
Finally, in \eqref{tsk_eq_4}, we used  Lemma \ref{useful_lemma}. 
Recalling that the loss function is bounded on $[0,1]$, we face with sub-Gaussian variables with parameter $\sigma=0.5$.
By setting
$\aTsk=2(\lenZm-1)/\lenZm$ (where $\aTsk\leq 1/2\sigma^2$) in Lemma \ref{useful_lemma}, and recalling that $\sigma=0.5$, from  \eqref{diff_1_tsk}, 
we found \eqref{tsk_eq_4}.

Similarly, inserting  \eqref{up_env}
into 
\eqref{diff_1} , we find that 
 \begin{align}
   \mathbbm E_{\mathrm P_{\taskRV_{1:\lenDataSet}}} \mathbbm E_{\mathrm P_{\trainEnvSimp}|\mathrm P_{\taskRV_1:\taskRV_\lenDataSet}} \left(    \Upsilon_{\textup{env}}\right)^2 
  & =
  \mathbbm E_{\mathrm P_{\taskRV_{1:\lenDataSet}}} \mathbbm E_{\mathrm P_{\trainEnvSimp}|\mathrm P_{\taskRV_1:\taskRV_\lenDataSet}} \left( \mathbbm E_{\PriorEnvSimp}e^{\theta_{\textup{env}}
\funcEnv[
\expLossEnvRV
]{\decomMetaRVTotRV}}\right)^2
   \label{env_eq_1}\\
   &\leq
   \mathbbm E_{\mathrm P_{\taskRV_{1:\lenDataSet}}} \mathbbm E_{\mathrm P_{\trainEnvSimp}|\mathrm P_{\taskRV_1:\taskRV_\lenDataSet}}\mathbbm E_{\PriorEnvSimp} \left( e^{(\lenDataSet-1)\left(
\expLossEnvRV
-\decomMetaRVTotRV\right)}\right)^2
   \label{env_eq_2}  \\
  &= \mathbbm E_{\mathrm P_{\taskRV_{1:\lenDataSet}}} \mathbbm E_{\mathrm P_{\trainEnvSimp}|\mathrm P_{\taskRV_1:\taskRV_\lenDataSet}}\mathbbm E_{\PriorEnvSimp}  e^{2(\lenDataSet-1)\left(
\expLossEnvRV
-\decomMetaRVTotRV\right)}
   \label{env_eq_3}\\
   &= \mathbbm E_{\PriorEnvSimp} \mathbbm E_{\mathrm P_{\taskRV_{1:\lenDataSet}}} \mathbbm E_{\mathrm P_{\trainEnvSimp}|\mathrm P_{\taskRV_1:\taskRV_\lenDataSet}} e^{2(\lenDataSet-1)\left(
\expLossEnvRV
-\decomMetaRVTotRV\right)}
   \label{env_eq_4}\\
  & \leq \lenDataSet
     \label{env_eq_5}
 \end{align}
 where in \eqref{env_eq_2}, since $a^2$ is a convex function, we applied Jensen's inequality.
 In  \eqref{env_eq_3}, we used the fact that $\exp(a)^b=\exp(a.b)$. Since the prior is independent of the data, by interchanging the order of expectations over $\mathrm P_{\taskRV_{1:\lenDataSet}}\mathrm P_{ \trainEnvSimp|\taskRV_{1:\lenDataSet}}$ and prior, we obtained \eqref{env_eq_4}. Finally, in  \eqref{env_eq_5}
 we used  Lemma \ref{useful_lemma}. 
By setting
$\aEnv=2(\lenDataSet-1)/\lenDataSet$ (where $\aEnv\leq 1/2\sigma^2$) in Lemma \ref{useful_lemma}, and recalling that $\sigma=0.5$, from  \eqref{diff_1_env}, 
we found \eqref{env_eq_5}.
 
Inserting \eqref{tsk_eq_4} and \eqref{env_eq_5} into \eqref{dif_1}, we have
 \begin{align}
 \mathbbm E_{\mathrm P_{\taskRV_{1:\lenDataSet}}} \mathbbm E_{\mathrm P_{\trainEnvSimp}|\mathrm P_{\taskRV_1:\taskRV_\lenDataSet}}\left(\Upsilon_{\textup{tsk}}^{\frac{1}{\lenDataSet}}\cdot
     \Upsilon_{\textup{env}}
\right)   
\leq 
 \sqrt{\lenZm^2\lenDataSet}=\lenZm\sqrt{\lenDataSet}  \label{dif_2}.
 \end{align}
Next, we focus on the affine transformation. 
Since from $2(\lenZm-1)(a-b)^2\leq c_{\textup{tsk}}$ (resp. $(\lenDataSet-1)(a-b)^2\leq c_{\textup{env}}$), we can conclude that 
$a\leq b+ \sqrt{ c_{\textup{tsk}}/(2(\lenZm-1))}$ (resp. $a\leq b+ \sqrt{ c_{\textup{env}}/(\lenDataSet-1)}$).

Inserting \eqref{dif_2} into \eqref{cor_tup_1_new_mc}, and applying affine transformation, we conclude the proof.
 
\section{Presenting New PAC-Bayes Bounds }
\label{aux_app_for_new_bnd}
\begin{theorem}
\label{the_sqt_for_k_markov}
 Under the setting of Theorem \ref{prop-gen-pac-tsk_two_KL}, for $k\in\mathbbm N=\{1,2,...\}$ and $\lenDataSet\geq 2$, the meta-generalization gap is bounded by
      \begin{align}
    \left |\mathbbm E_{U\sim\hyperpost} \left[
   \expLossEnvRV- \empLossEnvRV \right] \right|
   \leq
   \sqrt{\frac{(\lenDataSet-\lenDataSet^{\frac{1}{2k}})+2(\lenZm-\lenZm^{\frac{1}{2k}})}{2(\lenDataSet-\lenDataSet^{\frac{1}{2k}})(\lenZm-\lenZm^{\frac{1}{2k}})}}
   \hspace{10em}\nonumber\\
   \sqrt{2\Dkl\left(\hyperpost||\PriorEnvSimp  \right)
+\mathbbm E_{\hyperpost}\left[ \frac{1}{\lenDataSet}\sum_{i=1}^\lenDataSet \Dkl\left(\PWcondZimUsimp||\PriorTskSimp \right)\right]
+\log\frac{(\sqrt{\lenDataSet}\cdot \lenZm)^{\frac{1}{2}-\frac{1}{4k}}}{\delta}}
.
\label{sqt_for_k_markov}
 \end{align}
 \end{theorem}
 \begin{proof}
 We set  $\theta_{\textup{env}}=\theta_{\textup{tsk}}=1$,
 $\funcEnv[a]{b}=(\lenDataSet-\lenDataSet^{\frac{1}{2k}})(b-a)^2$ and $\funcTsk[a]{b}=2(\lenZm-\lenZm^{\frac{1}{2k}})(b-a)^2$. To bound the log-term, we use Lemma \ref{useful_lemma}, and in \eqref{diff_11_tsk} and \eqref{diff_11_env}, we set $\sigma=0.5$, $\aTsk=2-2\lenZm^{-1+1/(2k)}$ and $\aEnv=2-2\lenDataSet^{-1+1/(2k)}$.  
 By following exactly the same steps presented in the proof of Theorem \ref{the_Nw_McAll}, we conclude the proof. 
 \end{proof}
 
 \begin{theorem}
 \label{st_markov_th}
  Under the setting of Theorem \ref{prop-gen-pac-tsk_two_KL}, for $\lenDataSet\geq 2$, the meta-generalization gap is bounded by
           \begin{align}
    \left |\mathbbm E_{U\sim\hyperpost} \left[
   \expLossEnvRV- \empLossEnvRV \right] \right|
   \leq
   \sqrt{\frac{0.5\lenDataSet+\lenZm}{0.5\lenDataSet\cdot\lenZm}}
   \hspace{20em}\nonumber\\
   \sqrt{\Dkl\left(\hyperpost||\PriorEnvSimp  \right)
+\mathbbm E_{\hyperpost}\left[ \frac{1}{\lenDataSet}\sum_{i=1}^\lenDataSet \Dkl\left(\PWcondZimUsimp||\PriorTskSimp \right)\right]
+\log\frac{2\sqrt{2}}{\delta}
}.\label{st_markov}
 \end{align}
 \end{theorem}
  \begin{proof}
 We set  $\theta_{\textup{env}}=\theta_{\textup{tsk}}=1$,
 $\funcEnv[a]{b}=0.5\lenDataSet(b-a)^2$ and $\funcTsk[a]{b}=\lenZm(b-a)^2$. To bound the log-term, we use Lemma \ref{useful_lemma}, and in \eqref{diff_11_tsk} and \eqref{diff_11_env}, we set $\sigma=0.5$, $\aTsk=1$ and $\aEnv=1$.  
 By following exactly the same steps presented in the proof of Theorem \ref{the_Nw_McAll}, we conclude the proof. 
 \end{proof}
 We can apply different  $\funcEnv[a]{b}$ and $\funcTsk[a]{b}$ functions, and obtain different bound.

\section{General Lemmas}
\label{gen_lemma}
In this appendix, we provide a number of general lemmas that will be used throughout the paper.

\begin{lemma}
\label{Lemma_sum}
Let $X_i$ for  $i=1,...,\lenDataSet$ be independent random variables. Suppose  that for given $a_i\in\mathbbm R^+$, and  measurable  function $f_i$
\begin{align}
    \mathbbm P_{X_{i}}[f_i(X_i)\geq a_i]\leq \delta_i,\label{2020.08.05_19.45}
\end{align}
where $\delta_i\in[0,1]$. Then, 
\begin{align}
    \mathbbm P_{X_{1:\lenDataSet}}\left[\sum_i f_i(X_i)\leq \sum_i a_i\right]\geq 1-\sum_i\delta_i.\label{2021.02.21_23.11}
\end{align}
\end{lemma}
\begin{proof}
Firstly, we show that
\begin{align}
   \underbrace{ \left\{(x_1,...,x_\lenDataSet):\; \sum_i f_i(x_i)\geq  \sum_i a_i\right\}}_{\mathcal A}\subseteq\underbrace{\bigcup_{i=1}\{x_i:\; f_i(x_i)\geq a_i\}}_{\mathcal B}.\label{2021.02.21_23.12}
\end{align}
Let $(x_1,...,x_\lenDataSet)\notin \mathcal B$. It means that for all $i=1,...,\lenDataSet$, $f_i(x_i)<a_i$ which leads that $\sum_i f_i(x_i)<\sum_i a_i$, or $(x_1,...,x_\lenDataSet)\notin \mathcal A$. Thus, $\mathcal B^c\subset \mathcal A^c$, or equivalently $\mathcal A\subset \mathcal B$.

Next, from \eqref{2021.02.21_23.12}, one can conclude
\begin{align}
    \mathbbm P_{X_{1:\lenDataSet}}\left[\sum_i f_i(X_i)\geqslant \sum_i a_i\right ]\leq
    \sum_i     \mathbbm P_{X_{1:\lenDataSet}}\left[  f_i(X_i)\geqslant  a_i\right ]=    \sum_i     \mathbbm P_{X_{i}}\left[  f_i(X_i)\geqslant  a_i\right ]
    \leq\sum_i\delta_i \label{2020.08.05_19.58}
\end{align}
where the last inequality follows from
\eqref{2020.08.05_19.45}.  The proof can be concluded from \eqref{2020.08.05_19.58}.

\end{proof}

\begin{lemma}
\label{prior-indep}
Let $X_1,...,X_m$ be independent random variables, and   $g:\mathcal X\rightarrow\mathbbm R$ be a sub-Gaussian  function with parameter $\sigma$.
Assume $\Delta\triangleq\mathbbm E[g(X)]-\frac{1}{m}\sum_{k=1}^mg(X_i)$, where  for $\epsilon>0$, we have $\mathbbm P[\Delta\geq \epsilon ]\leq \exp(\frac{-m\epsilon^2}{2\sigma^2})$. Then 
\begin{align}
    \mathbbm E\left[e^{\lambda m \Delta^2}  \right]\leq \frac{1}{1-2\lambda\sigma^2}, 
\end{align}
for $\lambda\leq\frac{1}{2\sigma^2}$.
\end{lemma}
\begin{proof}
The proof is similar to Lemma 3 of  \cite{McAllaster1999}. For completeness, we repeat it again. 
Let $f^\star$ denotes the density function maximizing $\mathbbm E[e^{\lambda m \Delta^2} ]$ subject to the constraint that $\mathbbm P[\Delta\geq \epsilon ]\leq \exp(\frac{-m\epsilon^2}{2\sigma^2})$. The maximum occurs when $\mathbbm P_{f^\star}[\Delta\geq \epsilon ]= \exp(\frac{-m\epsilon^2}{2\sigma^2})$ leading to
\begin{align}
  f^\star(\Delta)=\frac{m\Delta}{\sigma^2} \exp\left( -\frac{m\Delta^2}{2\sigma^2} \right)\mathbbm{1} \{ \Delta\geq 0 \}.
\end{align}
Thus, we have
\begin{align}
      \mathbbm E\left[e^{\lambda m \Delta^2}  \right]\leq \int_{0}^{\infty}  \exp\left(\lambda m \Delta^2\right)\frac{m\Delta}{\sigma^2} \exp\left( -\frac{m\Delta^2}{2\sigma^2} \right) d\Delta =\frac{1}{1-2\lambda\sigma^2},\hspace{2em} \lambda<\frac{1}{2\sigma^2}.
\end{align}
\end{proof}

\begin{lemma}
\label{prior-indep-square}
Let $X_1,...,X_m$ be independent random variables. Assume $\Delta=\mathbbm E[g(X)]-\frac{1}{m}\sum_{k=1}^mg(X_i)$, where $g(\cdot)$ is sub-Gaussian  function with parameter $\sigma$. 
Then ,
\begin{align}
    \mathbbm E\left[e^{\lambda m \Delta^2}  \right]\leq \frac{1}{\sqrt{1-2\lambda\sigma^2}}, 
\end{align}
for $\lambda\leq\frac{1}{2\sigma^2}$.
\end{lemma}
\begin{proof}
The proof is similar to Theorem 2.6 \cite{Wainwright}. For completeness, we repeat it again.
Since $g(\cdot)$ is sub-Gaussian, we have
\begin{align}
    \mathbbm E\left[ e^{\lambda\Delta} \right]\leq \exp\left( \frac{\lambda^2\sigma^2}{2m} \right).\label{201.03.16_14.27}
\end{align}
Multiplying both sides of \eqref{201.03.16_14.27} by $\exp(-\frac{\lambda^2\sigma^2}{2sm})$ for $s\in(0,1)$, we find that
\begin{align}
    \mathbbm E\left[ e^{\lambda\Delta-\frac{\lambda^2\sigma^2}{2sm}} \right]\leq \exp\left( \frac{-\lambda^2\sigma^2}{2ms}(-s+1) \right).\label{201.03.16_14.29}
\end{align}
Next, we take integration with respect to $\lambda$. 
Since \eqref{201.03.16_14.29} is valid for any $\lambda\in\mathbbm R$, by using Fubini's theorem, we exchange the order of expectation and integration, leading to
\begin{align}
    \mathbbm E\left[ \exp\left(\frac{s m\Delta^2}{2\sigma^2}\right) \right]\leq \frac{1}{\sqrt{1-s}},\hspace{2em}\textup{for }0<s<1.
\end{align}
By defining $\lambda=\frac{s}{2\sigma^2}$, we conclude the proof.
\end{proof}
\begin{lemma}
\label{useful_lemma}
 Consider
$ \expLossTsk[w_i]{}$, $\empLossTsksimple[w_i]{}$,  $\metaGenLoss$, $   \mathrm L_{\trainEnv}(u) $ and $ \decomMeta[u]{\bm Z^{\lenZm}_i}{\task}$ defined by 
\eqref{exp_loss_tsk},
        \eqref{emp_loss_task}, \eqref{sharu_eq12}, \eqref{sharu_eq11} and \eqref{decom-pac}, respectively.
     Assume that the loss function $\ell(\cdot,\cdot)$ is bounded on the interval $[0,1]$, and hence it is sub-Gaussian with parameter $\sigma=(b-a)/2$. For $ \aEnv,\aTsk\leq {1}/{2\sigma^2}$, and data-free priors we have
     \begin{align}
         & \mathbbm E_{\mathrm P_{T_{1:\lenDataSet}}}\mathbbm E_{\mathrm P_{\trainEnvSimp|T_{1:\lenDataSet}}}  \mathbbm E_{   \PriorEnvSimp \PriorTskSimp} e^{\aTsk\lenZm( \mathrm L_{\mathrm P_{Z|\taskRV_i}}(W)-\empLossTskRV)^2}\leq
          \frac{1}{1-2\aTsk\sigma^2}
          ,\label{diff_1_tsk}\\
          &\mathbbm E_{\mathrm P_{T_{1:\lenDataSet}}}\mathbbm E_{\mathrm P_{\trainEnvSimp|T_{1:\lenDataSet}}}\mathbbm E_{\PriorEnvSimp}e^{\aEnv\lenDataSet\left(
\expLossEnvRV
-\decomMetaRVTotRV\right)^2}\leq
\frac{1}{1-2\aEnv\sigma^2}
,\label{diff_1_env}
     \end{align}
 and also     
        \begin{align}
         & \mathbbm E_{\mathrm P_{T_{1:\lenDataSet}}}\mathbbm E_{\mathrm P_{\trainEnvSimp|T_{1:\lenDataSet}}}  \mathbbm E_{   \PriorEnvSimp \PriorTskSimp}e^{\aTsk\lenZm( \mathrm L_{\mathrm P_{Z|\taskRV_i}}(W)-\empLossTskRV)^2}\leq  \frac{1}{\sqrt{1-2\aTsk\sigma^2}},
          \label{diff_11_tsk}\\
                 &   \mathbbm E_{\mathrm P_{T_{1:\lenDataSet}}}\mathbbm E_{\mathrm P_{\trainEnvSimp|T_{1:\lenDataSet}}}\mathbbm E_{\PriorEnvSimp}e^{\aEnv\lenDataSet\left(
\expLossEnvRV
-\decomMetaRVTotRV\right)^2}\leq
\frac{1}{\sqrt{1-2\aEnv\sigma^2}}
.\label{diff_11_env}
     \end{align}
\end{lemma}
\begin{proof}
We recall that since the prior is independent of the data, by interchanging the order of expectations over $\mathrm P_{\taskRV_{1:\lenDataSet}}\mathrm P_{ \trainEnvSimp|\taskRV_{1:\lenDataSet}}$ and priors
To show \eqref{diff_1_tsk}, and  \eqref{diff_11_tsk}, we note that $\mathrm P_{\taskRV_i\trainTsk}$ is the marginal distribution of  $ \mathrm P_{T_{1:\lenDataSet}}\mathrm P_{\trainEnvSimp|T_{1:\lenDataSet}} $. 
Since the priors are data-free, we have
        \begin{align}
          \mathbbm E_{\mathrm P_{T_{1:\lenDataSet}}}\mathbbm E_{\mathrm P_{\trainEnvSimp|T_{1:\lenDataSet}}}   \mathbbm E_{   \PriorEnvSimp \PriorTskSimp}&e^{\aTsk\lenZm( \mathrm L_{\mathrm P_{Z|\taskRV_i}}(W)-\empLossTskRV)^2}=
            \mathbbm E_{\mathrm P_{\taskRV_i\bm Z_i^\lenZm}} \mathbbm E_{   \PriorEnvSimp \PriorTskSimp}e^{\aTsk\lenZm( \mathrm L_{\mathrm P_{Z|\taskRV_i}}(W)-\empLossTskRV)^2}\\
           &=     \mathbbm E_{   \PriorEnvSimp \PriorTskSimp}\mathbbm E_{\mathrm P_{\taskRV_i\bm Z_i^\lenZm}}e^{\aTsk\lenZm( \mathrm L_{\mathrm P_{Z|\taskRV_i}}(W)-\empLossTskRV)^2}.
          \label{diff_1}
     \end{align}

Next,
we set  $\Delta=    \expLossTskRV-
     \empLossTskRV$, and $m=\lenZm$ 
in Lemma \ref{prior-indep}, and also Lemma \ref{prior-indep-square}, for $ \aTsk\leq \frac{1}{2\sigma^2}$,  we respectively  conclude that
\begin{align}
    & \mathbbm E_{\mathrm P_{\trainTskSimp}|\taskRV_i} e^{ 
     \aTsk\lenZm(
     \expLossTskRVRV-
     \empLossTskRV)^2}\leq \frac{1}{1-2\aTsk\sigma^2},\label{diff_2_a}\\
        &  \mathbbm E_{\mathrm P_{\trainTskSimp}|\taskRV_i} e^{ 
     \aTsk\lenZm(
     \expLossTskRVRV-
     \empLossTskRV)^2}\leq \frac{1}{\sqrt{1-2\aTsk\sigma^2}},\label{diff_2_b}
\end{align}
where by averaging both sides of \eqref{diff_2_a} and \eqref{diff_2_b} over $\taskRV_i$, from \eqref{diff_2_a} and \eqref{diff_2_b}, we find that
\begin{align}
    &   \mathbbm E_{   \PriorEnvSimp \PriorTskSimp}\mathbbm E_{\mathrm P_{\taskRV_i\bm Z_i^\lenZm}}e^{\aTsk\lenZm( \mathrm L_{\mathrm P_{Z|\taskRV_i}}(W)-\empLossTskRV)^2}\leq \frac{1}{1-2\aTsk\sigma^2},\label{diff_3_a},\\
        &    \mathbbm E_{   \PriorEnvSimp \PriorTskSimp}\mathbbm E_{\mathrm P_{\taskRV_i\bm Z_i^\lenZm}}e^{\aTsk\lenZm( \mathrm L_{\mathrm P_{Z|\taskRV_i}}(W)-\empLossTskRV)^2}\leq \frac{1}{\sqrt{1-2\aTsk\sigma^2}}.\label{diff_3_b}
\end{align}
Inserting the right hand-sides of \eqref{diff_3_a} and \eqref{diff_3_b} into \eqref{diff_1}, we conclude \eqref{diff_1_tsk}, and  \eqref{diff_11_tsk}.

Similarly, to show To show \eqref{diff_1_env}, and  \eqref{diff_11_env}, we use the fact that $    \expLossEnv[U]{}=\mathbbm E_{\mathrm P_{T\bm Z^\lenZm}}[  \decomMeta[U]{\bm Z^\lenZm}{\taskRV}]$.
Using the fact that prior is data-free, by setting $\Delta=  \mathbbm E_{\mathrm P_{T\bm Z^\lenZm}}[  \decomMeta[U]{\bm Z^\lenZm}{\taskRV}]
   -\decomMetaRVTotRV$, and $m=\lenDataSet$ 
in Lemma \ref{prior-indep}, and also Lemma \ref{prior-indep-square}, for $ \aEnv\leq \frac{1}{2\sigma^2}$,  we respectively  conclude that
\begin{align}
         &   \mathbbm E_{\PriorEnvSimp}   \mathbbm E_{\mathrm P_{T_{1:\lenDataSet}}}\mathbbm E_{\mathrm P_{\trainEnvSimp|T_{1:\lenDataSet}}}e^{\aEnv\lenDataSet\left(
  \mathbbm E_{\mathrm P_{T\bm Z^\lenZm}}[  \decomMeta[U]{\bm Z^\lenZm}{\taskRV}]
-\decomMetaRVTotRV\right)^2}\leq
\frac{1}{1-2\aEnv\sigma^2},\\
           & \mathbbm E_{\PriorEnvSimp}   \mathbbm E_{\mathrm P_{T_{1:\lenDataSet}}}\mathbbm E_{\mathrm P_{\trainEnvSimp|T_{1:\lenDataSet}}}e^{\aEnv\lenDataSet\left(
  \mathbbm E_{\mathrm P_{T\bm Z^\lenZm}}[  \decomMeta[U]{\bm Z^\lenZm}{\taskRV}]
-\decomMetaRVTotRV\right)^2}\leq
\frac{1}{\sqrt{1-2\aEnv\sigma^2}},
\end{align}
concluding \eqref{diff_1_env}, and  \eqref{diff_11_env}.
\end{proof}

\begin{lemma}
\label{kl_lem}
Let $X_1,...,X_n$ be i.i.d random variables, and   $f:\mathcal X\rightarrow[0,1]$ be a bounded function. For all $n>8$, we have
\begin{align}
    \mathbbm E\left[e^{n\Dkl\left( \frac{1}{n}\sum_i f(X_i)|| \mathbbm E[f(X)]  \right) } \right] \leq 2\sqrt{n}.
    \end{align}
\end{lemma}
\begin{proof}
 See Theorem 1 of \cite{Maurer}, then for $n\geq 8$, the right hand side of Eq. 5 of \cite{Maurer} is smaller than $\sqrt{n}$.
\end{proof}

\begin{lemma}
\label{D_gama}
Let $X_1,...,X_n$ be i.i.d random variables. For the given function  $f:\mathcal X\rightarrow[0,1]$, we have
\begin{align}
    \mathbbm E\left[e^{n\Dgama\left(\frac{1}{n}\sum_i f(X_i) ||\mathbbm E[f(X)] \right)}  \right]\leq 1,
\end{align}
where $\Dgama(a||b)=\gamma a-\log(1-b+be^\gamma)$.
\end{lemma}
\begin{proof}
 See Equation (18) of \cite{McAllester2013}.  For completeness, we repeat it again. Since for $a\in[0,1]$ and $\gamma\in\mathbbm R$, we have $e^{\gamma a}\leq 1-a+a\cdot e^\gamma$, we conclude $   e^{\frac{\gamma}{n}\sum_i f(X_i) }\leq 1-\frac{1}{n}\sum_i f(X_i)+e^\gamma  \frac{1}{n}\sum_i f(X_i) $, by taking expectation from both sides, we find that $\mathbbm E\left[e^{\frac{\gamma}{n}\sum_i f(X_i) }\right]\leq 1-\mathbbm E\left[ f(X) \right]+e^\gamma  \mathbbm E\left[ f(X) \right]$. Taking logarithm from both sides leads to
 \begin{align}
\mathbbm E\left[e^{\frac{\gamma}{n}\sum_i f(X_i) -\log\left( 1-\mathbbm E\left[ f(X) \right]+e^\gamma  \mathbbm E\left[ f(X) \right]\right)}\right]\leq 1 \label{2021.03.19_.8.47}
 \end{align}
 Now, since $X_i$s are independent
 \begin{align}
     \mathbbm E_{\mathrm P_{\bm X_{1:n}}}\Big [ e^{n\Dgama\left(\frac{1}{n}\sum_i f(X_i) ||\mathbbm E[f(X)] \right)}  \Big]&=
          \mathbbm E_{\mathrm P_{\bm X_{1:n}}}\left [ \prod_{i=1}^n e^{\Dgama\left(\frac{1}{n}\sum_i f(X_i) ||\mathbbm E[f(X)] \right)}  \right]
          \\
          &=\prod_{i=1}^n \mathbbm E_{\mathrm P_{X_i}}\left [  e^{\Dgama\left(\frac{1}{n}\sum_i f(X_i) ||\mathbbm E[f(X)] \right)}  \right]
          \\
&=\prod_{i=1}^n \mathbbm E_{\mathrm P_{X_i}}\left [  e^{\frac{\gamma}{n}\sum_i f(X_i) -\log(1-\mathbbm E[f(X)]+e^\gamma  \mathbbm E[f(X)])}  \right]\leq 1,\label{modi_13}
 \end{align}
 where the equality in \eqref{modi_13} and the last inequality follow from the definition of $\Dgama$  and \eqref{2021.03.19_.8.47}, respectively.
\end{proof}

\begin{lemma}
\label{lemma2_MCallester}
Let  $\Dgama(a||b)=\gamma a-\log(1-b+be^\gamma)$. For $\lambda>0.5$ and $\mathrm C\in \mathbbm R$, if $\mathrm D_{-\frac{1}{\lambda}}(a||b)<\mathrm C$, then
\begin{align}
    b\leq \frac{a+\lambda\mathrm C}{1-\frac{1}{2\lambda}}.
\end{align}
\end{lemma}
\begin{proof}
  See Lemma 2 of \cite{McAllester2013}. 
\end{proof}

As mentioned before, to find PAC-Bayes bounds, usually we have four steps, namely choosing a suitable convex function,
applying Jensen's, change of measure and Markov’s inequalities. For the most PAC-Bayesian proofs, Donsker-Varadhan’s
inequality is used as the change of measure inequality:
\begin{lemma}
\label{lem_dons}
For any measurable function $\phi(\cdot)$,  and two distributions $\Prior$ and $\Posterior$, we have
\begin{align}
    \mathbbm E_{\Posterior} [\phi(X)]\leq \Dkl(\Posterior||\Prior)+\log\left(
        \mathbbm E_{\Prior} \left[e^{\phi(X)}\right]
        \right).\label{Donsker}
\end{align}
\end{lemma}



\end{document}